\documentclass{article} 
\usepackage{graphicx}
\usepackage{algorithm,algorithmic}
\usepackage{amstext,amssymb,amsmath,amsthm,bm,bbm}
\usepackage{url}
\usepackage{boxedminipage}
\usepackage{hyperref}
\usepackage{wrapfig}
\usepackage{ifthen}
\usepackage{xcolor}
\usepackage{framed}
\usepackage{algorithmic,algorithm}
\usepackage{enumitem}
\usepackage{epstopdf}
\usepackage{mathtools}
\usepackage{caption}
\usepackage{fullpage, prettyref}
\usepackage{makecell}
\usepackage{natbib}

\allowdisplaybreaks
\DeclareMathOperator*{\argmin}{arg\,min}

\newtheorem{theorem}{Theorem}

\newtheorem{lemma}{Lemma}

\newtheorem{corollary}{Corollary}

\newtheorem{assumption}{Assumption}

\newtheorem{invariant}{Invariant}

        {\hspace*{\fill}$\Box$\par}

\newcommand{\R}{\mathbb{R}}
\newcommand{\abs}[1]{\left\lvert #1 \right\rvert}
\newcommand{\norm}[1]{\left\lVert #1 \right\rVert}

\newcommand{\iprod}[2]{\left\langle #1,#2 \right\rangle}
\newcommand{\twonorm}[1]{\left\lVert #1 \right\rVert_{2}}
\newcommand{\frob}[1]{\left\lVert #1 \right\rVert_{F}}
\newcommand{\ceil}[1]{\left\lceil #1 \right\rceil}
\newcommand{\floor}[1]{\lfloor #1 \rfloor}

\newcommand{\lprp}[1]{\left( #1 \right)}

\colorlet{shadecolor}{blue!20}

\newcommand{\Mt}[1][q]{M^{(#1)}}
\newcommand{\Lt}[1][q]{L^{(#1)}}
\newcommand{\Ct}[1][q]{C^{(#1)}}
\newcommand{\Ltt}[1][q]{\widetilde{L}^{(#1)}}
\newcommand{\Ctt}[1][q]{\widetilde{C}^{(#1)}}
\newcommand{\Ntt}[1][q]{\widetilde{N}^{(#1)}}
\newcommand{\Qt}[1]{Q^{(#1)}}
\newcommand{\St}[1]{S^{(#1)}}
\newcommand{\cs}{\mathcal{CS}}

\renewcommand{\Pr}[1]{\mathcal{P}_{#1}}
\newcommand{\C}{C}
\renewcommand{\L}{L}
\newcommand{\N}{N}
\newcommand{\M}{M}
\newcommand{\D}{D}
\newcommand{\Rm}{R}
\newcommand{\MoS}{\Mo_{\setminus S}}
\newcommand{\LoS}{\Lo_{\setminus S}}
\newcommand{\NoS}{\No_{\setminus S}}
\newcommand{\CoS}{\Co_{\setminus S}}
\newcommand{\U}{U}
\newcommand{\V}{V}
\newcommand{\US}{\U_{\setminus S}}
\newcommand{\VS}{\V_{\setminus S}}
\newcommand{\SigS}{\Sigma_{\setminus S}}
\newcommand{\Uo}{\U^*}
\newcommand{\Vo}{\V^*}
\newcommand{\So}{\Sigma^*}
\newcommand{\UoS}{\Uo_{\setminus S}}
\newcommand{\VoS}{\Vo_{\setminus S}}
\newcommand{\SigoS}{\So_{\setminus S}}
\newcommand{\Mo}{M^*}
\newcommand{\Do}{D^*}
\newcommand{\Lo}{\L^*}
\newcommand{\Co}{\C^*}

\newcommand{\No}{\N^*}

\newcommand{\Lot}[1]{\L^{(#1)}}
\newcommand{\Cot}[1]{\C^{(#1)}}

\newcommand{\Uot}[1]{\U^{(#1)}}
\newcommand{\Sigot}[1]{\Sigma^{(#1)}}
\newcommand{\Vot}[1]{\V^{(#1)}}

\newcommand{\order}[1]{O\left({#1}\right)}
\newcommand{\orpca}{\textbf{OR-PCA}}
\newcommand{\orpcan}{\textbf{OR-PCAN}}
\newcommand{\orpcag}{\textbf{OR-PCAG}}

\newcommand{\ncpca}{\textbf{TORP-N}}
\newcommand{\ncpcacl}{\textbf{TORP}}
\newcommand{\ncpcan}{\textbf{TORP-N}}
\newcommand{\ncpcag}{\textbf{TORP-G}}
\newcommand{\ncpcab}{\textbf{TORP-BIN}}
\newcommand{\fpr}{\textbf{FAST-PR}}
\newcommand{\supp}[1]{\textrm{Supp}\left(#1\right)}
\newcommand{\HT}[1]{\mathcal{HT}_{#1}}

\newcommand{\thresh}{\eta}
\newcommand{\nThresh}{\rho}

\newcommand{\colSp}{\alpha}
\newcommand{\nOut}{n_{\text{thres}}}

\newcommand{\pperpuo}{P_{\perp}^{\Uo}}
\newcommand{\pperpu}{P_{\perp}^{\U}}
\newcommand{\pperput}{P_{\perp}^{\Uot{t}}}
\newcommand{\pperputo}{P_\perp^{\Uot{t+1}}}
\newcommand{\pperpuok}{P_\perp^{\Uo_{1:k}}}
\newcommand{\pperpur}{P_\perp^{\U_{1:r}}}
\newcommand{\pperpy}{P_\perp^{Y}}
\title{Thresholding based Efficient Outlier Robust PCA}
\author{Yeshwanth Cherapanamjeri \qquad Prateek Jain \qquad Praneeth Netrapalli \\
Microsoft Research India\\
\texttt{\{t-yecher,prajain,praneeth\}@microsoft.com}
}

\begin{document}
\bibliographystyle{plainnat}
\maketitle
\begin{abstract}
We consider the problem of outlier robust PCA (\orpca) where the goal is to recover principal directions despite the presence of outlier data points. That is, given a data matrix $\Mo$, where $(1-\alpha)$ fraction of the points are noisy samples from a low-dimensional subspace while $\alpha$ fraction of the points can be arbitrary outliers, the goal is to recover the subspace accurately. Existing results for \orpca~have serious drawbacks: while some results are quite weak in the presence of noise, other results have runtime quadratic in dimension, rendering them impractical for large scale applications.

In this work, we provide a novel thresholding based iterative algorithm with per-iteration complexity at most linear in the data size. Moreover, the fraction of outliers, $\alpha$, that our method can handle is tight up to constants while providing nearly optimal computational complexity for a general noise setting. For the special case where the inliers are obtained from a low-dimensional subspace with additive Gaussian noise, we show that a modification of our thresholding based method leads to significant improvement in recovery error (of the subspace) even in the presence of a large fraction of outliers.
\end{abstract}

\section{Introduction}
Principal Component Analysis (PCA) is a critical first step for any typical data exploration/analysis effort and is widely used in a variety of applications. A key reason for the success of PCA is that it can be performed efficiently using Singular Value Decomposition (SVD).

However, due to various practical reasons like measurement error, presence of anomalies etc., a large fraction of data points can be corrupted in a somewhat correlated and even adversarial manner. Unfortunately, SVD is fragile with respect to outliers and can lead to arbitrarily inaccurate principal directions in the presence of even a small number of outliers. So, designing an outlier robust PCA (\orpca) algorithm is critical for several application domains. 

Formally, the setting of \orpca~is as follows: given a data matrix $\Mo=\Do+\Co\in \R^{d\times n}$ where  $D^*=[x_1, \dots, x_n]$ corresponds to $n$ clean ``inlier'' data points and $\Co$ has at most $\alpha$-fraction of non-zero columns that can corrupt the corresponding clean data points arbitrarily, the goal of \orpca~is to estimate the principal components of $\Do$ accurately, i.e., recover $\Uo\in \R^{d\times r}$, the top-$r$ left singular vectors of $D^*$. 

Vanilla SVD does not do the job since the top singular vectors of $\Mo$ can be arbitrarily far from $\Uo$ if the operator norm of $\Co$ ($\|\Co\|_2$) is large, as can be the case when $\alpha=\Omega(1)$. Any algorithm trying to solve \orpca~needs to exploit the column sparsity of $\Co$ to obtain a better estimate of $\Uo$. In particular, they need to find $S^*= \supp{\Co}$---$\supp{A}$ is the index of non-zero columns of $A$---so that $\Uo$ can be estimated using top singular directions of $\Mo_{\setminus{S^*}}$, i.e., columns of $\Mo$ restricted to complement set of $S^*$.

Existing results for \orpca~fall into two categories based on: a) Nuclear norm \citep{DBLP:journals/tit/XuCS12,DBLP:journals/tit/ZhangLZ16}, b) iterative PCA \citep{DBLP:journals/tit/XuCM13}. Nuclear norm based approaches work with exactly same setting as ours, but require $O(nd^2)$ computational time which is prohibitive for typical applications.
Iterative PCA based techniques require $O(n^2d)$ computation which in general is significantly higher than our algorithms. Moreover, these results do not recover the exact principal directions even if the inliers are restricted to a low-dimensional subspace and just a constant number of outliers are present. 

Our approach is based on solving the following natural optimization problem: $$\orpca:\ \ \min_{\D, \C \in \R^{d\times n}} \|\Mo-\D-\C\|_F^2 \text{ s.t. } rank(\D)\leq r,\ \ |\supp{\C}|\leq \alpha n.$$
{\em Technical Challenges}: The main challenges with \orpca~are its non-convexity and combinatorial structure which rules out standard tools from convex optimization. Furthermore, due to the column-sparsity constraint standard SVD based techniques also do not apply. Instead, we propose a simple iterative method that constructs an estimate of outliers $C$ and using that, an estimate of the inliers $\Mo-C$. SVD of the estimated inliers is then used to obtain an estimate of the principal directions. Now, a significant challenge is to use these principal directions to re-estimate outliers. There are two different scenarios here: 
\begin{itemize}
	\item	\emph{Length of all the outliers is smaller than the smallest singular value of the inliers and hence do not stand out}: In this case however, the principal directions are not much affected by outliers. So, outliers can be recognized by taking a projection on to the orthogonal space to the estimated principal directions. As we get better estimate of principal directions, we can be more aggressive in determining the outliers. 
	\item	\emph{Length of at least one of the outliers is larger than the small singular values of inliers}: In this case again, length based thresholding fails since the lengths of the inliers are dominated by the larger singular values. Similarly, the above mentioned thresholding scheme also fails as some of the estimated principal directions will be heavily biased towards those outliers. A key and somewhat surprising algorithmic insight of our work is that:  as some of the estimated singular vectors are heavily affected by outliers, projection of such outliers on these spurious singular vectors will be inordinately high. So, in contrast to the above thresholding operator, we can use the length of projection of points along estimated principal directions as well to detect and threshold outliers. 
\end{itemize}
In both the scenarios, we can identify more outliers and repeat the procedure till convergence.
A key assumption we make, in order to ensure that inliers are not thresholded, is that most of the inliers have ``limited'' influence on the principal directions, i.e., the data matrix is incoherent (see Assumption~\ref{as:rank}). Such an assumption holds for various typical settings, for example when inliers are noisy and uniform samples from a low-dimensional subspace.  

\noindent{\em Contributions}: \textit{The main contribution of our work is to show that, under some regularity conditions, it is indeed possible to solve \orpca~near optimally, in essentially the same time as that taken by vanilla PCA.} In particular, we propose  a thresholding based approach that iteratively estimates inliers using {\em two} different thresholding operators and then use SVD to estimate the principal directions. We also show that our method recovers $\Uo$ nearly optimally and efficiently as long as the fraction of corrupted data points (column-support of $\Co$) is less than $O(1/r)$ where $r$ is the dimensionality of principal subspace $\Uo$. More concretely, we study the problem in three settings:
\begin{enumerate}
\item {\em Noiseless setting}: $\Mo=\Do+\Co$ where the clean data matrix $D^*$ is a rank-$r$ matrix with $\mu$-incoherent right singular vectors (see Assumption~\ref{as:rank}) and $\Co$ has at most $\alpha\cdot n$ non-zero columns. For this setting, we design a novel \textbf{T}hresholding based \textbf{O}uliter \textbf{R}obust \textbf{P}CA algorithm (\ncpcacl) that recovers  $\Uo$ up to an error of $\epsilon$ in time $\order{ndr\log \frac{n\twonorm{\Mo}}{\epsilon}}$, if $\alpha\leq \frac{1}{128\mu^2 r}$. Note that this is essentially the time taken for vanilla PCA as well. In contrast, existing results for the same setting require $O(\frac{nd^2}{\epsilon^2})$\footnote{Dependence on $\epsilon$ is due to the standard rates of gradient descent when applied to the non-smooth non-strongly convex optimization problem given in \eqref{eq:nuc}; however, using more refined RSC-style analysis, $\epsilon$ dependency might be improved but we are not aware of such an existing result.} computation to recover $\Uo$. Note that the number of outliers our results can handle (i.e., $\alpha \leq 1/128 \mu^2 r$) is optimal up to constant factors in the sense that if $\alpha > 1/\mu^2 r$, then there exists a matrix $\Mo$ which has more than one decomposition satisfying the above conditions.
\item {\em Arbitrary noise}: In the second setting, $\Mo=\Do+\Co$ where the clean data matrix $\Do$ can be written as $\Lo + \No$ and $\Lo$, the rank-$r$ projection of $\Do$ has $\mu$ incoherent right singular vectors. $\Co$ on the other hand, again has at most $\alpha\cdot n$ non-zero columns. If $\alpha\leq \frac{1}{128\mu^2 r}$, our proposed algorithm \ncpcan~ guarantees recovery of $\Uo$ (left singular vectors of $\Lo$) up to $\order{\sqrt{r} \frob{\No} + \epsilon}$ error, in $\order{ndr\log \frac{n\twonorm{\Mo}}{\sqrt{r} \frob{\No} + \epsilon}}$ time. Again this is essentially the same as the time taken for vanilla PCA up to $\log$ factors. In contrast, existing results for this problem get stuck at a significantly larger error of ($\sqrt{n} \frob{\No}$), with a runtime of $O(\frac{nd^2}{\epsilon^2})$, which is slower than ours by a factor of $d$. 
\item {\em Gaussian noise}: In this setting, we again have $\Mo=\Do+\Co$ where the clean data matrix $D^*$ is a sum of low-rank matrix $\Lo$ and a Gaussian noise matrix $\No$, i.e., each element of $\No$ is sampled independently and identically (iid) from $\mathcal{N}(0,\sigma^2)$. \ncpcag, which is our proposed algorithm for this special case, recovers $\Uo$ up to an error of $\order{\sqrt{r\log d}\twonorm{\No}}$. This not only improves upon the result we obtained for the arbitrary noise case above, which is $\order{\sqrt{r} \frob{\No}}$, it also improves significantly upon the existing results. However, in order to achieve this improvement in error, we require $ n > d^2$ and the algorithm has a runtime of $\order{n^2 d}$. 
\end{enumerate}
To summarize, our results obtain stronger guarantees in terms of both final error and runtime, while being able to handle a large number of outliers, for three different settings of \orpca. Moreover, in the first two settings, our run time matches that of standard PCA up to $\log$ factors. Please refer Tables~\ref{tab:comp1} and~\ref{tab:comp2} for comparison of our results with existing results.
{
	\begin{table}[h]
		\begin{center}
			\begin{tabular}{| c | c | c |}
				\hline
				& \thead{Run time} & \thead{Error ($\frob{(I - \U\U^{\top})\Lo}$)} \\ \hline
				\cite{XuCS12}	& $\order{\frac{d^2n}{\epsilon^2}}$ & $\order{\sqrt{n} \frob{\No} + \epsilon}$ \\ \hline
				\ncpcan & $\order{dnr \log \frac{n \twonorm{\Mo}}{\epsilon}}$ & $\order{\sqrt{r} \frob{\No} + \epsilon}$ \\ \hline
			\end{tabular} \\
			\caption{Arbitrary noise: comparison of our results and existing results for the arbitrary noise setting. Here, $\Mo = \Do + \Co$, where $\Do = \Lo + \No$ is the clean data matrix with $\Lo$, the rank-$r$ projection of $\Do$ having $\mu$-incoherent right singular vectors (Assumption~\ref{as:rank}) and $\Co$ has $O(\frac{1}{\mu^2 r})$ is a column sparse matrix. The error is measured as the residual of $\Lo$ when projected on to the estimated space $\U$. Note that we obtain better error and better runtime compared to existing results. Also, the noiseless case is a special case with $\No = 0$.}
			\label{tab:comp1}
			\begin{tabular}{| c | c | c |}
				\hline
				& \thead{Run time} & \thead{Error($\frob{(I - \U\U^{\top})\Lo}$)} \\ \hline
				\cite{XuCS12} & $\order{\frac{n d^2 r}{\epsilon^2}}$ & $\order{\sigma n \sqrt{d}+\epsilon}$ \\ \hline
				\ncpcan & $\order{ndr^2 \log {\frac{\twonorm{\Mo}}{\epsilon}}}$ & $\order{\sigma \sqrt{nd} + \epsilon}$ \\ \hline
				\ncpcag & $\order{n^2 d r \log {\frac{\twonorm{\Mo}}{\epsilon}}}$ & $\order{\sigma \sqrt{n \log (d)}+\epsilon}$ \\ \hline
			\end{tabular}
		\caption{Gaussian noise: comparison of existing results with ours for the Gaussian noise setting: $\Mo = \Do + \Co$, where inliers $\Do = \Lo + \No$ with $\Lo$ being a rank-$r$ matrix with $\mu$-incoherent right singular vectors (Assumption~\ref{as:rank}) and $\No$ is a Gaussian matrix with variance $\sigma^2$.  Specializing \ncpcan~ for this setting already gives us faster and better results than existing ones. \ncpcag~ further improves the error by a factor of $\sqrt{d/\log d}$. The algorithm however requires extra $O(n)$ factor in the runtime.}
		\label{tab:comp2}\vspace*{-20pt}
		\end{center}
	\end{table}
}

\textbf{Paper Outline}: The paper is organized as follows. We will review related work in Section~\ref{sec:related}. We then present a formal definition of the problem in Section~\ref{sec:form} and our main results in Section~\ref{sec:results}. We then present our algorithm for each of the three settings: a) noise-less setting, b) arbitrary data, c) Gaussian noise, in Section~\ref{sec:orpcacl}, \ref{sec:orpca}, \ref{sec:orpcag}, respectively. We provide a brief overview of our proofs in Section~\ref{sec:po}. Finally, we conclude with a few open problems and promising future directions in Section~\ref{sec:conc}. 
\subsection{Related Works}
\label{sec:related}
In this section, we will discuss related work and compare existing results with ours. Existing theoretical results for OR-PCA fall into two categories:

a) The first category of approaches, more in-line with our own work, are based on Outlier-Pursuit (\cite{DBLP:journals/tit/XuCS12}) which optimizes a convex relaxation of the OR-PCA problem where the rank and column sparsity constraints are replaced by the trace-norm (sum of singular values) and $\norm{.}_{2, 1}$ (sum of the column lengths) penalties. That is, they solve the following optimization problem:
\begin{equation}\label{eq:nuc}
	\text{(Outlier Pursuit): } \min \norm{L}_{*} + \lambda \norm{C}_{2, 1}\ s.t\ \frob{M - L + C} \leq \frob{\No}.
\end{equation}
While outlier pursuit obtains optimal recovery guarantees in absence of any noise, a main drawback is that its computational complexity is quadratic in $d$, i.e., $O(nd^2)$. Moreover, in presence of noise the bounds given in \cite{DBLP:journals/tit/XuCS12} are $O(\sqrt{n})$ worse than our result. Extensions of Outlier-Pursuit to the partially observed setting (\cite{DBLP:journals/tit/ChenXCS16}) and online setting \cite{DBLP:conf/nips/FengXY13} have also been proposed but share the drawback of high computational complexity. Recently, \cite{DBLP:journals/tit/ZhangLZ16} showed that Outlier Pursuit achieves recovery even with the fraction of outliers larger than the information theoretic lower bound. However, this requires the outliers to be ``well-spread'' which in practice is restrictive; our results allow the corruption matrix to be constructed in adversarial manner although the corruptions cannot depend on the Gaussian noise in the setting (3) described in previous section. 

b) The second line of approaches based on HR-PCA (\cite{DBLP:journals/tit/XuCM13}) iteratively prune or reweigh data points which have a large influence on the singular vectors and select an estimate with the Robust Variance Estimator metric. When applied to our finite sample setting, these results cannot achieve exact recovery even in the noiseless case with a single outlier. Moreover, their running time in this setting is $\order{n^2dr}$ while ours is nearly linear in the input size. In the special case of Gaussian noise, these results incur at least a constant error $\order{\sigma_1(\Lo)}$ whereas our recovery guarantee scales linearly with the standard deviation of the noise $\order{\sqrt{r} \sigma\log d}$, achieving exact recovery when $\sigma = 0$. \cite{DBLP:conf/icml/FengXY12} propose a deterministic variant of HR-PCA and \cite{DBLP:conf/icml/YangX15} extend HR-PCA to PCA-like algorithms like Sparse PCA and Non-Negative PCA.

There has been much recent work on the related problem of Robust PCA (\cite{CandesLMW11}, \cite{NIPS2014_5430}, \cite{DBLP:conf/nips/YiPCC16}, \cite{DBLP:journals/corr/CherapanamjeriG16}). In contrast to the setting considered here, the corruptions are assumed to be both row and column sparse i.e., unlike our setting no data can be corrupted in all its dimensions. This restriction allows stronger recovery guarantees but makes the results inapplicable to the setting of outlier robust PCA. 

\subsection{Notations}
We use the following notations in this paper. For a vector $v$, $\norm{v}$ and $\twonorm{v}$ denote the $\ell_2$ norm of $v$. For a matrix $\M$, $\norm{\M}$ and $\twonorm{\M}$ denote the operator norm of $\M$ while $\frob{\M}$ denotes the Frobenius norm of $\M$. $\sigma_k(\M)$ denotes the $k^{\textrm{th}}$ largest singular value of $\M$. SVD refers to singular value decomposition of a matrix. $\mathcal{SVD}_r(\M)$ refers to the rank-$r$ SVD of $\M$. Given a matrix $\M$, $\M_i$ denotes the $i^{\textrm{th}}$ column of $\M$ while $\M_{i,:}$ denotes the $i^{\textrm{th}}$ row of $M$. Given a matrix $\M \in \R^{d\times n}$ and a set $S\subseteq [n]$, $\M_S$ is defined as
\begin{align*}
	(\M_S)_i = \left\{ \begin{array}{ll}
	\M_i &\mbox{ for } i \in S,\\
	0 &\mbox{ otherwise.}
	\end{array}\right.			
\end{align*}
$\M_{\setminus S}$ denotes $\M_{[n]\setminus S}$. $\supp{M}$ denotes column support of $M$, i.e., the set of indices of non-zero columns of $M$. 

We use two hard -thresholding operators in this paper.
Given a matrix $\Rm$, the first hard-thresholding operator, $\mathcal{HT}_\nThresh \lprp{\Rm}$ denotes the set of indices $j$ of the top $\nThresh$ fraction of largest columns (in $\ell_2$ norm) in $\Rm$. Given a matrix $\Rm$, the second hard-thresholding operator $\widetilde{\HT{\zeta}}(\Rm)$ is defined as,
\begin{equation}
\label{eqn:httDef}
\widetilde{\HT{\zeta}}(N)=\{i:\ \text{s.t. } \|N_{i}\|_2 \geq \zeta\}.
\end{equation}
$\pperpu(M)$ denotes $(I-UU^\top)M$. For any set $\mathcal{S} \in \mathbb{R}^{d}$, we will use $\Pr{\mathcal{S}}$ to denote the projection onto the set $\mathcal{S}$. We will also use $\Pr{U}$ to denote the projection onto the column space of $U$ for $U \in \mathbb{R}^{d \times r}$. 

\section{Problem Formulation}
\label{sec:form}
In this section, we will formally present the setting of the paper. We are given $\Mo = \Do + \Co \in \R^{d\times n}$, where columns of $\Do$ are inliers and $\Co$ are outliers. Only an $\alpha$ fraction of the points are outliers, i.e., only $alpha$ fraction of the columns of $\Co$ are non-zero. Broadly, we consider three scenarios:
\begin{itemize}
	\item	\textit{\orpca} (Noiseless setting): The points in $\Do$ lie \textit{entirely} in a low-dimensional subspace i.e., $\Do=\Lo$ is a rank-$r$ matrix.
	\item	\textit{\orpcan} (Noisy setting): The points in $\Do$ lie \textit{approximately} in a low-dimensional subspace i.e., $\Do = \Lo + \No$ where $\Lo$ is the rank-$r$ projection of $\Do$ and $\No$ is the noise matrix.
	\item	\textit{\orpcag} (Gaussian noise setting): The points in $\Do$ come from a low-dimensional subspace with additive Gaussian noise i.e., $\Do = \Lo + \No$ where $\Lo$ is a rank-$r$ matrix and $\No$ is a Gaussian noise matrix i.e., each element of $\No$ is sampled iid from $\mathcal{N}(0,\sigma^2)$.
\end{itemize}
In all the above settings, the goal is to find the low dimensional subspace spanned by the columns of $\Lo$.

This problem is in general ill-posed. Consider for instance the case, when most of the true data points $\Lo$ are zero and only an $\alpha$ fraction of them are non-zero. These points can either be considered inliers or outliers. In order to overcome this issue, standard assumption used in literature~\citep{DBLP:journals/tit/XuCS12, DBLP:conf/nips/FengXY13} is that of incoherence. Also, incoherence is satisfied in several standard settings; for example, when the inliers are noise and uniform samples from a low-dimensional subspace. 
\begin{assumption}
  \label{as:rank}
  \textbf{Rank and incoherence of $\Lo$: } $\Lo\in \mathbb{R}^{d\times n}$ is a rank-$r$ incoherent matrix, i.e., $\twonorm{e_{i}^{\top}V^*} \leq \mu \sqrt{\frac{r}{n}} \ \forall i\in [n]$, where $\Lo = U^*\Sigma^* (V^*)^{\top}$ is the SVD of $\Lo$.
\end{assumption}


\section{Our Results}\label{sec:results}
In this section, we will present our results for the three settings mentioned above.
\subsection{\orpca~-- Noiseless Setting}
Recall that in the noiseless setting, we observe $\Mo = \Do + \Co$ where $\Do$ is a rank-$r$, $\mu$-incoherent matrix corresponding to clean data points and the column-support of $\Co$ is at most $\alpha n$. 
The following theorem is our main result for this setting.
\begin{theorem}[Noise-less Setting]
	\label{thm:thmCln}
	Let $\Mo, \Do$ and $\Co$ be as described above. If $\alpha\leq \frac{1}{128 \mu^2 r}$, then Algorithm~\ref{alg:ncpcacl} run with parameters $\nThresh = \frac{1}{128\mu^2 r}$ and $T=\log \frac{10 n \|\Mo\|_2}{\epsilon}$, returns a subspace $U$ such that,  
	\begin{equation*}
	\frob{(I - UU^\top) D^*} \leq \epsilon.
	\end{equation*}
\end{theorem}
\textbf{Remarks}:
\begin{itemize}
	\item	Note that the guarantee of Theorem~\ref{thm:thmCln} can be right away converted to a bound on the subspace distance between $U$ and that spanned by the columns of $\Do$. In particular, we obtain $\frob{(I - UU^\top)\Uo} \leq \epsilon/\sigma_r(\Do)$, where $\sigma_r(\Do)$ denotes the smallest singular value of $\Do$ and $\Uo$ contains the singular vectors of $\Do$.
	\item	Since the most time consuming step in each iteration is computing the top-$r$ SVD of an $n\times d$ matrix, the total runtime of the algorithm is $\order{ndr\log \frac{10 n \|\Mo\|_2}{\epsilon}}$.
	\item	The above assumption on the column sparsity of $\Co$ is tight up to constant factors i.e., we may construct an incoherent matrix $\Lo$ and column sparse matrix $\Co$ such that it is not possible to recover the true column space of $\Lo$ when the column sparsity of $\Co$ is larger than $\frac{1}{\mu^2 r}$.
\end{itemize}
\subsection{\orpcan~-- Arbitrary Noise}
We now consider the noisy setting. Here we observe $\Mo = \Do + \Co$, where $\Do$ is a near low rank matrix i.e., $\Do = \Lo + \No$ where $\Lo$ is the best rank $r$ approximation to $\Do$ and is a $\mu$ incoherent matrix, while $\No$ is a noise matrix. $\Co$ is again column sparse with at most an $\alpha$ fraction of the columns being non-zero.
\begin{theorem}[Arbitrary Noise]
	\label{thm:thmN}
Consider the setting above. If $\alpha\leq \frac{1}{128 \mu^2 r}$, then Algorithm~\ref{alg:ncpca} when run with parameters $\rho=\frac{1}{128\mu^2 r}$, $\eta=2\mu \sqrt{\frac{r}{n}}$ and $T=\log \frac{20\|\Mo\|_2\cdot n}{\epsilon}$ iterations, returns a subspace $U$ such that:
	\begin{equation*}
	\frob{(I - UU^\top) L^*} \leq 60 \sqrt{r} \frob{\No} + \epsilon.
	\end{equation*}
\end{theorem}
\textbf{Remarks}:
\begin{itemize}
	\item	The theorem shows that up to $\sqrt{r}\frob{\No}$ error, recovered directions $U$ contains top $r$ principle directions of inliers. We do not optimize the constants in our proof. In fact, we obtain a stronger result in Theorem~\ref{thm:masterN} which for certain regime of noise $\No$ can lead to significantly better error bound. Note that when there is no noise i.e., $\No = 0$, we recover Theorem~\ref{thm:thmCln}.
	\item	The guarantee here can again be converted to a bound on subspace distance. For instance, for any $k \leq r$, we have $\frob{(I - \U\U^{\top})\Uo_{[k]}}\leq \left(60 \sqrt{r} \frob{\No}+ \epsilon\right)/\sigma_k(\Lo)$, where $\Uo_k$ denotes the top-$k$ left singular subspace of $\Lo$ and $\sigma_k(\Lo)$ denotes the $k^{\textrm{th}}$ largest singular value of $\Lo$.
	\item	The total runtime of the algorithm is $\order{ndr^2 \log \frac{10\|\Mo\|_2\cdot n}{\epsilon}}$. However, the outer loop over $k$ in Algorithm~\ref{alg:ncpca} can be replaced by a binary search for values of $k$ between $1$ and $r$. This reduces the runtime to $\order{ndr \log r\log \frac{10 \|\Mo\|_2\cdot n}{\epsilon}}$. See Algorithm~\ref{alg:ncpcab} for more details. 
\end{itemize}

\subsection{\orpcag -- Gaussian Noise}
We now consider the Gaussian noise setting. Here we observe $\Mo = \Do + \Co$, where $\Do$ is a near low rank matrix i.e., $\Do = \Lo + \No$ where $\Lo$ is a rank-$r$, $\mu$ incoherent matrix, while $\No$ is a Gaussian matrix with each entry sampled iid from $\mathcal{N}(0,\sigma^2)$. $\Co$ is again column sparse with at most an $\alpha$ fraction of the columns being non-zero.

\begin{theorem}[Gaussian Data]
	\label{thm:thmG}
Consider the setting mentioned above. Suppose $\alpha\leq \frac{1}{1024\mu^2 r}$. Then, Algorithm~\ref{alg:ncpcag} stops after at most $T=\alpha n$ iterations and returns a subspace $U$ such that:  
	\begin{equation*}
	\|(I - UU^\top) L^*\|_2\leq 4\sqrt{\log d}\|N^*\|_2.
	\end{equation*}
with probability at least $1-\delta$ as long as $n \geq \frac{16\mu^2r^2d}{c_1} \left[\log\lprp{\frac{1}{3\delta}} + d\log(80d)\right]$ for some absolute constants $c_1$ and $c_2$.
\end{theorem}

\textbf{Remarks}:
\begin{itemize}
	\item	Data points coming from a low-dimensional subspace with additive Gaussian noise is a standard statistical model that is used to justify PCA. Though this can be seen as a special case of arbitrary noise model, we get a much tighter bound than that obtained from Theorem~\ref{thm:thmN}.
	\item	While Theorem~\ref{thm:thmN} gives an asymptotic error bound of $\frob{(I - UU^\top) L^*} \leq 60 \sqrt{r} \frob{\No}$, Theorem~\ref{thm:thmG} gives an asymptotic error bound of $\frob{(I - UU^\top) L^*} \leq 4  \twonorm{\No}$. Note that the right hand sides above refer to Frobenius and operator norms respectively.
	\item	The improvement mentioned above is obtained by carefully leveraging the fact that Gaussian random vectors are spread uniformly in all directions and that there is a small fraction of vectors which is correlated. However, in order to make this argument, we need $n = \order{d^2}$. It is an open problem to get rid of this assumption.
	\item	Note also that our result is tight in the sense that as $\sigma \rightarrow 0$, we recover the result of Theorem~\ref{thm:thmCln}. However, the running time of the algorithm is $O(n^2 d)$ which is significantly worse than that of \ncpcacl. We leave design/analysis of a more efficient algorithm that achieves similar error bounds as Theorem~\ref{thm:thmG} as an open problem. 
	\item 	We can obtain the above result even when each column in $\No$ is drawn from a sub-Gaussian distribution rather than each entry being iid $\mathcal{N} (0, \sigma^2)$.
\end{itemize}


\section{Outlier Robust PCA: Noiseless Setting}\label{sec:orpcacl}
In this section, we present our algorithm \ncpcacl(Algorithm~\ref{alg:ncpcacl}) that applies to the special case of noise-less data, i.e., when  $M^*=\Lo+\Co$, $\Lo$ is rank-$r$, $\mu$-incoherent matrix. While restrictive, this setting allows us to illustrate the main ideas behind our algorithm approach and the analysis techniques in a relatively simpler fashion.
\label{sec:algCl}

Recall that the goal is to estimate $\Uo$, the left singular vectors of $\Lo$. However, SVD of $\Mo$ can lead to singular vectors arbitrary far from $\Uo$, because a few column of $\Co$ can be so large that they can bias entire singular vectors in their direction. 

Our algorithm instead tries to exploit two key structural properties of the problem: sparsity of $\Co$ and incoherence of $\Lo$. Our algorithm maintains a column-sparse estimate $\Cot{t}$ of $\Co$. Each iteration of the algorithm computes a low-rank approximation of an estimate of the inliers $\Mo-\Cot{t}=\Lo+\Co-\Cot{t}$. Note that if $(I-\Uo(\Uo)^\top)\Co=(I-\Uo(\Uo)^\top)\Cot{t}$, then left singular vectors of $\Mo-\Cot{t}$ will be $\Uo$. 


Our next step finds {\em residual} length of each column $\Mo_i$ when projected on to the orthogonal subspace to $\Uot{t}$. If length of each outlier is smaller compared to the smallest singular value of $\Lo$, then using sparsity of $\Cot{t}$ and $\Co$, we can show that $\Uot{t}$ is "close" to $\Uo$ in all directions. So, the residual of some of the outliers will stand out and those columns can be removed. This is achieved by the hard-thresholding step 5, 8 of Algorithm~\ref{alg:ncpcacl}. 

A big challenge in this scheme is that if a column of the perturbation matrix $\Co-\Cot{t}$ is "very" long compared to smaller singular values of $\Lo$, then they can  perturb some directions of $\Uo$ significantly. This will lead to a failure of the above thresholding approach. However, in such a case, some of the columns of $\Co-\Cot{t}$ will be close to a few spurious singular vectors in $\Uot{t}$ (our current estimate of $\Uo$). Hence, projection of such outliers along $\Uot{t}$ will be inordinately long. On the other hand, due to incoherence of $\Lo$, inliers' projection along $\Uot{t}$ can be bounded in magnitude. So, we can safely threshold out certain outliers. Steps 6, 8 of Algorithm~\ref{alg:ncpcacl} perform this thresholding operation.  

In summary, our algorithm computes low-rank approximation of $\Mo-\Cot{t}$ and uses the obtained singular vectors $\Uot{t}$ to threshold out a few columns of $\Cot{t}$ to obtain next estimate $\Cot{t+1}$ of $\Co$. See Algorithm \ref{alg:ncpcacl} for a pseudo-code of our approach. 

{\em Time Complexity}: Note that the computationally most expensive operation in each iteration is that of SVD which requires $O(ndr)$ time. So, the overall time complexity of the algorithm is $O(ndr\cdot T)$. As we show in Section~\ref{sec:po}, as long as $\Co$ is column-sparse, $T\approx \log \frac{1}{\epsilon}$ suffices to obtain an $\epsilon$ approximation to $\Uo$. So, the overall complexity of the algorithm is $O(ndr\cdot \log\frac{20\|\Mo\|_2}{\epsilon})$. Note that typically SVD computation is approximate, while all three of our algorithms and analyses assumes exact SVD. However, extension of our analysis to allow for small additive error is straightforward and we ignore it in favor of simplicity and readability. 

{\em Parameters}: The algorithm requires an estimate of rank $r$ and threshold parameter $\rho$ which in turn depends on estimate of incoherence $\mu$ of $\Lo$. We propose to set these parameters via cross-validation. Note that setting rank to be any value larger than rank of $\Lo$ will lead to recovery of $\Uo$, as long as $\Co$ is sparse enough. Similarly, if estimation of $\mu$ is larger than incoherence of $\Lo$, then it only effects number of corrupted columns in $\Co$ that can be allowed. So, a simple cross-validation approach with appropriately chosen grid-size leads to recovery of $\Uo$ as long as $\Co$ is sparse enough (as specified in Theorem~\ref{thm:thmCln}). 
\begin{algorithm}[!h]
	\caption{Thresholding based Outlier Robust PCA ($\ncpcacl$)}
	\label{alg:ncpcacl}
	\begin{algorithmic}[1]
		\STATE \textbf{Input}: Data $\Mo \in \mathbb{R}^{d\times n}$, Target rank $r$, Threshold fraction $\nThresh$, Number of iterations $T$
		
		\STATE $\Ct[0] \leftarrow 0$
		
		\FOR{Iteration $t = 0$ to $t = T$}
		\STATE $[\Uot{t}, \Sigma^{(t)}, \Vot{t}] \leftarrow \mathcal{SVD}_r\lprp{\Mo - \Ct[t]};\; \Lt[t] \leftarrow \Uot{t} \Sigma^{(t)} (\Vot{t})^{\top}$ \hspace{1em}{\color{blue}\rlap{\smash{$\left.\begin{array}{@{}c@{}}\end{array} \right\}%
				{\small \begin{tabular}{c}Projection onto \\ space of \\ low rank matrices \end{tabular}}$}}}

%
		\STATE $R \leftarrow (I - \Uot{t}(\Uot{t})^\top)\Mo$ {\em /* Compute residual */}
		\STATE $E \leftarrow (\Sigma^{(t)})^{-1} (\Uot{t})^\top \Mo${\em /* Compute incoherence */}
		\STATE $\cs^{(t + 1)} \leftarrow \mathcal{HT}_\nThresh \lprp{R} \cup \mathcal{HT}_{\nThresh} \lprp{E} $ \hspace{14em}{\color{red}\rlap{\smash{$\left.\begin{array}{@{}c@{}}\\{}\\{}\\{}\end{array} \right\}%
				{\small \begin{tabular}{c}Projection onto \\ space of \\ column sparse \\ matrices \end{tabular}}$}}}
		\STATE $\Ct[t + 1] \leftarrow \Mo_S $\\{\em /* Threshold points with high coherence or high residual*/}
		\ENDFOR
		
		\STATE $[U, \Sigma, V] \leftarrow \mathcal{SVD}_r\lprp{\Mo - \Ct[T + 1]}$
		
		\STATE \textbf{Return: }$U$
	\end{algorithmic}
\end{algorithm}




\section{Outlier Robust PCA: General Noise}\label{sec:orpca}
In this section, we introduce our algorithm for the general case of Outlier Robust PCA with arbitrary inlier data $D^*=\Lo+\No$, i.e., the noise matrix $\No$ is arbitrary. Recall that the goal is to recover left singular vectors of $\Lo$. 
\label{sec:algGen}

Our algorithm for the general \orpca\ problem builds upon the \ncpcacl\ algorithm but with added complexity due to the presence of noise matrix $\No$. That is the algorithm alternately updates estimate of the outliers $\Cot{t}$ and the principal direction $\Uot{t}$ using two thresholding  operators along with SVD. 
However due to noise $\No$, our estimate of $\Uot{t}$ gets perturbed furthermore leading to arbitrary perturbation of the singular vectors of $\Uo$ corresponding to smaller eigenvalues of $\Lo$ and hence cannot be recovered. To alleviate this concern, our \ncpca\ algorithm proceeds via a pair of nested loops: 
\begin{enumerate}[leftmargin=*,labelindent=\widthof{\textbf{Outer Iteration on k:}}]
	\item[\textbf{Outer Iteration on k:}] The outer loop iterates over the rank-variable $k$ which represents the rank of the principal subspace we wish to estimate.
	\item[\textbf{Inner Iteration on t:}] The inner loop iteratively revises estimates of the principal subspace and a set of outliers until a stopping criteria is triggered.
\end{enumerate}
Intuitively, as in Algorithm~\ref{alg:ncpcacl}, each inner iteration of Algorithm~\ref{alg:ncpca} obtains a better estimate of $\Co$ and top-$k$ singular components of $\Lo$. That is, the $k$-th outer iteration after several of such inner iterations estimates $\Uo$ up to $\approx \sigma_k(\Lo)$. But when the noise $\|\No\|_F$ becomes comparable to the $k^{th}$ singular value of $\Lo$, then the algorithm terminates (Line 14, Algorithm~\ref{alg:ncpca}) as at that point it may not be possible to estimate the remaining singular vectors of $\Lo$. As we don't know $\|\No\|_F$ explicitly, we detect this event based on the number of data points which have a large influence on the estimated singular vectors (see lines 11, 12 of Algorithm~\ref{alg:ncpca}). 

Roughly, our stopping criterion allows us to make two statements regarding the termination of the algorithm:
\begin{enumerate}[leftmargin=*]
	\item When the algorithm terminates, the outlier columns that we have not thresholded will only have small influence on the estimated principal vectors. This is because all points with large influence will be thresholded before the estimate is computed.
	\item The algorithm will not terminate if $\sigma^*_k > > \frob{\No}$: The bound on $\frob{\No}$ ensures that not many inlier points can have large influence on estimate of the $k^{th}$ singular vector.
\end{enumerate}
By using the above two claims, our analysis shows that \ncpca\ recovers $\Uo$ up to $\sim \|\No\|_F$ error. 

{\em Time Complexity}: Time complexity of each inner iteration of \ncpca\ is $O(ndk)$. Hence, overall time complexity is $O(ndr^2)$, as $k$ can be as large as $r$. However, using a slightly more complicated algorithm and analysis (see Algorithm~\ref{alg:ncpcab}), we can search for appropriate $k$ using binary search, so the time complexity of the algorithm can be improved to $O(ndr\log r)$. 

{\em Parameter Estimation}: The algorithm requires $3$ parameters: rank $r$, threshold $\rho$ which depends on incoherence $\mu$ of $\Lo$ and expressivity parameter $\eta$. We can search for these parameters using a coarse-grid search as estimates of these parameters up to constants are enough for our algorithm to succeed albeit with a slightly stricter restriction (by constant factors) on the number of corrupted data points. 

        
        



\begin{algorithm}[!ht]
  \caption{Thresholding based Noisy Outlier Robust PCA ($\ncpcan$)}
  \label{alg:ncpca}
  \begin{algorithmic}[1]
    \STATE \textbf{Input}: Corrupted matrix $\Mo \in \mathbb{R}^{d\times n}$, Target rank $r$, Expressivity parameter $\thresh$, Threshold fraction $\nThresh$, Inner iterations $T$

    \FOR{$k = 1$ to $k = r$}
    \STATE $\Ct[0] \leftarrow 0$, $\tau \leftarrow false$
      \FOR{$t = 0$ to $t = T$}
        \STATE $[\Uot{t}, \Sigot{t}, \Vot{t}] \leftarrow \mathcal{SVD}_k\lprp{\Mo - \Ct[t]}$, $\Lt[t] \leftarrow \Uot{t}\Sigot{t}(\Vot{t})^\top$ \hspace{1em}{\color{blue}\rlap{\smash{$\left.\begin{array}{@{}c@{}}\\{}{}\end{array} \right\}%
        \begin{tabular}{c}Projection onto space of \\ low rank matrices \end{tabular}$}}}
        \vspace*{1.5em}
        \STATE $E \leftarrow (\Sigot{t})^{-1} (\Uot{t})^\top \Mo$ {\em /* Compute Incoherence */}
        \STATE $R \leftarrow (I - \Uot{t}(\Uot{t})^\top) \Mo$ {\em /* Compute residual */} 
        \hspace{5.7em}{\color{red}\rlap{\smash{$\left.\begin{array}{@{}c@{}}\\{}\\{}\\{}\end{array} \right\}%
        		\begin{tabular}{c}Projection onto space of \\ column sparse matrices \end{tabular}$}}}
        \STATE $\cs^{(t + 1)} \leftarrow \mathcal{HT}_{2\nThresh} \lprp{\Mo, E} \cup \mathcal{HT}_\nThresh \lprp{\Mo, R}$ 
        \STATE $\Ct[t + 1] \leftarrow \Mo_{\cs^{(t + 1)}}$

        \STATE $\nOut \leftarrow \abs{\{i: \norm{E_{i}} \geq \thresh\}}$ {\em /* Compute high incoherence points */}
        \STATE $\tau \leftarrow \tau \vee (\nOut \geq 2\nThresh n)$ {\em /* Check termination conditions */}
      \ENDFOR
      \IF{$\tau$}
        \STATE $break$
      \ENDIF
      \STATE $[U, \Sigma, V] \leftarrow \mathcal{SVD}_k\lprp{\Mo - \Ct[T + 1]}$
    \ENDFOR

    \STATE \textbf{Return: }$U$
  \end{algorithmic}
\end{algorithm}



\section{Outlier Robust PCA: Gaussian Noise}\label{sec:orpcag}
In this section, we present our algorithm for the special case of the Outlier Robust PCA problem when inlier points are generated using a standard Gaussian noise model. That is, when $D^*=\Lo+\No\in \R^{d\times n}$ where each entry of the noise matrix $\No$ is sampled i.i.d. from ${\cal N}(0, \sigma^2)$. Our result for arbitrary $\No$ (Theorem~\ref{thm:thmN}) estimates $\Uo$ up to $\sim \|\No\|_F$ error, which is $\Omega(\sigma \sqrt{dn})$ for Gaussian noise. However, using a slight variant of Algorithm~\ref{alg:ncpca} and exploiting the noise structure, Algorithm~\ref{alg:ncpcag} is able to estimate $\Uo$ up to $\sigma \sqrt{n\log d}$ error, which is better than the previous one by a factor of $\order{d/\log d}$. 


\begin{algorithm}[!ht]
	\caption{Thresholding based Outlier Robust PCA with Gaussian Noise ($\ncpcag$)}
	\label{alg:ncpcag}
	\begin{algorithmic}[1]
		\STATE \textbf{Input}: Corrupted matrix $\Mo \in \mathbb{R}^{d\times n}$, Target rank $r$, Incoherence Parameter $\mu$, Noise Level $\sigma$
		\STATE $\M \leftarrow \Mo$, $\tau \leftarrow true$
		\STATE $\zeta_1 \leftarrow \sigma \lprp{\frac{5}{4} \mu \sqrt{r} + d^{\frac{1}{2}} + 2 d^{\frac{1}{4}} \sqrt{\log\lprp{\frac{\mu^2 r}{c_2}}}}$, $\zeta_2 \leftarrow \sigma \sqrt{2r} \lprp{\frac{5}{4} \mu + 2 \sqrt{\log \lprp{\frac{\mu^2r^2d}{c_1}}}}$
		\STATE $\Ct[0] \leftarrow 0$, $\cs^{(0)} \leftarrow \{\}$, $\cs^{(-1)} \leftarrow \{0\}$, $t \leftarrow 0$
		
		\WHILE{$\cs^{(t)} \neq \cs^{(t - 1)}$}
		\STATE $[\Uot{t}, \Sigot{t}, \Vot{t}] \leftarrow \mathcal{SVD}_{r + 1} (\Mo - \Ct[t])$, $\ \ \Lt[t] \leftarrow \Uot{t}\Sigot{t}(\Vot{t})^\top$ \hspace{0em}{\color{blue}\rlap{\smash{$\left.\begin{array}{@{}c@{}}\\{}\end{array} \right\}%
				\begin{tabular}{c}Projection onto\\ space of \\ low rank matrices \end{tabular}$}}} 
		\vspace{1em}
		\STATE $\mathcal{E}^{(t)} \leftarrow \{x: x = \Uot{t}\Sigot{t} y \text{ for some } \norm{y} \leq 2\mu \sqrt{r / n}\}$ 
		
		\STATE $\widehat{L}^{(t)} \leftarrow \Pr{\mathcal{E}^{(t)}}(\Lt[t])$ {\em /* Projection onto incoherent matrices */}
		\STATE $\mathcal{I} \leftarrow \left\{i: \norm{\Lt[t]_{i} - \widehat{L}^{(t)}_{i}} > \zeta_2\right\}$ {\em /* Points with large influence */}
		
		\STATE $\cs^{(t + 1)} \leftarrow \cs^{(t)} \cup \widetilde{\HT{\zeta_1}} \lprp{\Lt[t] - \widehat{L}^{(t)}}$ \\ {\em /* Updating support of outliers */}
		
		\IF{$\abs{I} \geq \frac{24nc_1}{\mu^2dr}$}
		\STATE $\cs^{(t + 1)} \leftarrow \cs^{(t + 1)} \cup \widetilde{\HT{\zeta_2}} \lprp{\Lt[t] - \widehat{L}^{(t)}}$ \hspace{5em}{\color{red}\rlap{\smash{$\left.\begin{array}{@{}c@{}}\\{}\\{}\\{}\\{}\\{}\\{}\\{}\end{array} \right\}%
				\begin{tabular}{c}Projection onto space of \\ column sparse matrices \end{tabular}$}}} \\ {\em /* Update support of outliers */}
		\ENDIF
		\STATE $\Ct[t + 1] \leftarrow \Mo_{\cs^{(t + 1)}}$ \\ {\em /* Compute Sparse Projection */}
		\STATE $t \leftarrow t + 1$
		\ENDWHILE
		\STATE \textbf{Return: } U
	\end{algorithmic}
\end{algorithm}
%

At a high level the philosophy of our \ncpcag\  algorithm is similar to \ncpcacl, i.e., we iteratively revise estimate of $\Co$ and the top singular vectors $\Uo$ using SVD and thresholding. That is, we iteratively threshold columns of $\Mo$, that we estimate are corrupted. However, due to Gaussian noise structure our thresholding step is significantly different than that of \ncpcacl\ or \ncpca. 

In particular, the choice of our thresholding criteria (see lines 10, 12 of Algorithm~\ref{alg:ncpcag}) uses the following two insights:
\begin{enumerate}[leftmargin=*,labelindent=\labelwidth+\labelsep]
	\item Length based thresholding (with respect to $\zeta_1$---line 10 of Algorithm~\ref{alg:ncpcag}): This thresholding step is used to ensure that the noise in each data point is at most $\order{\sigma \sqrt{d}}$. As the length of random Gaussian vector is at most $\order{\sigma \sqrt{d}}$ with high probability, only a small number of inliers are thresholded in this step (Lemma~\ref{lem:gauLthConc}).
	\item Projection based thresholding (with respect to $\zeta_2$---line 12 of Algorithm~\ref{alg:ncpcag}): In this step, we threshold points that have large projection along the estimated principal subspace. Note that out of $n$ columns of $\No$, at most $\order{\frac{1}{\mu^2 r d}}$ fraction of points have projected lengths greater than $\order{\sigma \sqrt{\log (d)}}$ along any direction (Lemma~\ref{lem:gauipConc}). Thus, chances of a inliers being thresholded in this step is low. On the other hand, any outlier that heavily influences a principal direction will be thresholded by this step.
\end{enumerate}
Algorithm~\ref{alg:ncpcag} provides a detailed pseudo-code of \ncpcag. Step 6 of the algorithm computes \textbf{rank}-$\mathbf{(r+1)}$ SVD of the estimate of inlier matrix $\Mo-\Co$. Step 7 defines a set of vectors, whose projection onto singular vectors of $\Lot{t}$ is ``typical'' for an inlier which is composed of a low-dimensional point perturbed by Gaussian noise vector of length $O(\sigma \sqrt{d})$. 

This set is used in step 10 to threshold outliers using the hard-thresholding operator $\widetilde{\HT{\zeta}}$ as defined in~\eqref{eqn:httDef}. Next, the set $\mathcal{I}$ consists of points which have a large influence on the estimated principal components. In the absence of outliers, the size of this set is bounded by $\frac{12 n c_1}{\mu^2 d r}$ with high probability. A large deviation in the size of this set indicates the presence of of outliers and the entire set is thresholded.

Note on Approximate Computation: We would like to note that the projection operator defined in step~8 of the algorithm can be computed efficiently to arbitrary accuracy. A pseudo-code for computing the required projection can be found in Algorithm~\ref{alg:fpr}. Algorithm~\ref{alg:fpr} reduces the problem to the univariate problem of finding the root of a monotonically decreasing function in a bounded interval which can be found efficiently via binary search. For the sake of simplicity, we assume that the projection step and the $\mathcal{SVD}$ are computed exactly. Our analysis can be extended to the case where the projection and $\mathcal{SVD}$ are computed approximately with some added technical difficulty. 

\section{Proof Overview}\label{sec:po}
In this section, we provide a brief overview of our analysis for the three main results. 
\subsection{Noiseless Setting---Theorem~\ref{thm:thmCln}}

%
\label{sec:anCl}
In this section, we present the proof of Theorem~\ref{thm:thmCln}.
Recall that we are given $\Mo = \Do + \Co$, where $\Do=\Lo$ is a rank-$r$, $\mu$-incoherent matrix and $\Co$ has at most a fraction of $\rho$ non-zero columns. We can assume with out loss of generality that $\Do$ and $\Co$ have disjoint column supports as we can rewrite $\Mo_i$, for $i\in \supp{\Co}$, as $\Mo_i = \Do_i+\Co_i=0+(\Co_i+\Do_i)$ thus absorbing $\Do_i$ in $\Co_i$ itself. 

Our proof consists of three main steps. Given any set of columns $S$ and letting $[U_{\setminus S},\Sigma_{\setminus S}, V_{\setminus S}]$ be the top-$r$ SVD of $\Mo_{\setminus S}$, we establish the following:
\begin{enumerate}[leftmargin=*,labelindent=\labelwidth+\labelsep]
	\item[\textbf{Step 1:}] Every non-zero column of $\Do$ has significantly smaller residual when projected onto subspace orthogonal to $U_{\setminus S}$ than the norm of corrupted columns of $\Mo_{\setminus S}$ (Lemma~\ref{lem:clResCl}), so its likelihood of being thresholded (Line 6, 8 of Algorithm~\ref{alg:ncpcacl}) is small, 
	\item[\textbf{Step 2:}] Every non-zero column of $\Do$ has small incoherence with respect to $[U_{\setminus S},\Sigma_{\setminus S}, V_{\setminus S}]$ (Lemma~\ref{lem:clExpCl}), i.e., its projection onto $U_{\setminus S}$ cannot be ``too large''. Hence, its likelihood of being thresholded (Line 7,8 of Algorithm~\ref{alg:ncpcacl}) is also small, 
	\item[\textbf{Step 3:}] Any non-zero column of $\Co$ which has small residual and incoherence compared to those of a non-zero column of $\Do$ and hence won't be thresholded by Algorithm~\ref{alg:ncpcacl}, has small residual when projected onto $\Uo$ . That is, the column itself is close to subspace spanned by $\Uo$ and hence does not effect estimation of $\Uo$ (Proof of Theorem~\ref{thm:thmCln}).
\end{enumerate}
That is, either a corrupted column will be thresholded or it is close to $\Uo$ while inliers ($\Do$) have little likelihood of being thresholded (step 1,2 above).
We now present the formal statements and their proofs. We start with two lemmata establishing Steps 1,2 above. Detailed proofs of the lemmata are given in Appendix~\ref{sec:pfclResCl} and \ref{sec:pfclExpCl}, respectively. 

\begin{lemma}
	\label{lem:clResCl}
Consider the setting of Theorem~\ref{thm:thmCln}.
Let $S \subset [n]$ denote a subset of columns of $\Mo$ such that $\abs{S} \leq 2 \nThresh n$. Let $\MoS$ ($\LoS$) be obtained from $\Mo$ ($\Lo$) by setting the columns corresponding to indices specified in $S$ to $0$. Let $\US \SigS (\VS)^\top$ ($\UoS \SigoS (\VoS)^\top$) be the rank-$r\ \mathcal{SVD}$ of $\MoS$ ($\LoS$), then $\forall i$:
	\begin{equation*}
		\norm{(I - \US (\US)^\top) \Lo_{i}} \leq \frac{33}{32} \mu \sqrt{\frac{r}{n}} \norm{(I - \Uo {(\Uo)}^\top) \MoS}
	\end{equation*}
\end{lemma}
\begin{lemma}
	\label{lem:clExpCl}
Under the setting of Lemma~\ref{lem:clResCl}, we have for every $i$:
	\begin{equation*}
	\norm{\SigS^{-1} \US^\top \Lo_{i}} \leq \frac{33}{32} \mu \sqrt{\frac{r}{n}}.
	\end{equation*}
\end{lemma}

We now present the proof of Theorem \ref{thm:thmCln} where we illustrate Step 3:
\begin{proof}
	We will start by showing the quantity $\frob{(I - U^*(U^*)^\top) \Mt[t+1]}$ decreases at a geometric rate, where $\Mt[t+1]=\Mo-\Cot{t+1}$. Let $\Qt{t}$ denote the columns of $\Co$ that are not thresholded in iteration $t$. Also let $\St{t}$ denote the columns of $\Lo$ that are thresholded in iteration $t$. Let $\Ltt[t+1] \coloneqq \Lo_{\setminus \St{t}}$, $\Ctt[t+1] \coloneqq \Co_{\Qt{t}}$, and $P_{\perp}^{\U}(M)=(I-\U(\U)^\top)M$. Then, we have:
	\begin{align}
		&\frob{P_{\perp}^{\Uo}(\Mt[t+1])}^2  = \frob{P_{\perp}^{\Uo}(\Ltt[t+1] + \Ctt[t+1])} ^2 = \frob{P_{\perp}^{\Uo}(\Ctt[t+1])}^2 \notag\\&= \sum\limits_{j \in \Qt{t}} \norm{ \pperpuo(\Uot{t} \Sigot{t} W^{(t )}_j + R^{(t )}_j)}^2 
		 \leq 2 \sum\limits_{j \in \Qt{t}} \norm{\pperpuo(\Uot{t}) \Sigot{t}W^{(t)}_j}^2 + \norm{R^{(t)}_j}^2,\label{eq:cl1}
	\end{align}
	where $W^{(t)}_j=(\Sigot{t})^{-1}(\Uot{t})^T\Co_j$ and $R^{(t)}_j=\pperput(\Co_j)$, $\forall j\in \Qt{t}$. 
	The last inequality follows from triangle inequality and the fact that $(a + b)^2 \leq 2(a^2 + b^2)$. 

	Recall, that we threshold a particular column $l$ in iteration $t$ based on $\norm{\pperput(\Mo_{l})}$ and $\norm{(\Sigot{t})^{-1} (\Uot{t})^\top \Mo_{l}}$. For a particular $j \in \St{t}$ that wasn't thresholded in iteration $t$, we know that there exists a column $i_j$ such that $\norm{(\Sigot{t})^{-1} (\Uot{t})^\top \Lo_{i_j}} \geq \norm{(\Sigot{t})^{-1} (\Uot{t})^\top \Co_{j}}$. Similarly, there exists a column $k_j$ such that $\norm{\pperput(\Lo_{k_j})} \geq \norm{\pperput( \Co_{j})}$. From Lemmas \ref{lem:clExpCl} and \ref{lem:clResCl}, we have:
	\begin{equation}\label{eq:cl2}
		\norm{W^{(t)}_j} \leq \frac{33}{32} \mu \sqrt{\frac{r}{n}} \qquad \norm{R^{(t)}_j} \leq \frac{33}{32} \mu \sqrt{\frac{r}{n}} \norm{\pperpuo(\Mt[t])}
	\end{equation}
	Using \eqref{eq:cl1} and \eqref{eq:cl2}, we have: 
	\begin{multline*}
		\frob{P_{\perp}^{\Uo}(\Mt[t+1])}^2\leq 2 \sum\limits_{j \in \St{t}} \lprp{\frac{33}{32}}^2 \mu^2 \frac{r}{n} \norm{\pperpuo(\Uot{t})\Sigot{t}}^2 + \lprp{\frac{33}{32}}^2 \mu^2 \frac{r}{n} \norm{\pperpuo(\Mt[t])}^2\\ 
		\leq 4 \cdot \frac{9}{8} \cdot \frac{\mu^2 r}{n}\cdot \rho n \cdot \|\pperpuo(\Mt[t])\|^2\leq \frac{1}{4} \norm{\pperpuo(\Mt[t])}^2, 
	\end{multline*}
	where second last inequality follows from $|\St{t}|\leq \rho n$ and the last inequality follows from $\rho\leq \alpha \leq \frac{1}{128 \mu^2 r}$. 
	By recursively applying the above inequality, we obtain:
	\begin{equation}\label{eq:cl3}
		\frob{\pperpuo( \Mt[T + 1])} \leq \frac{\epsilon}{20n}.
	\end{equation}
Also, note that using variational characterization of SVD, we have $\|\pperpu(\Mt[T+1])\|_F\leq \frob{\pperpuo( \Mt[T + 1])}$. Theorem now follows from the following argument: 
	\begin{multline*}
		\frob{\pperpu(\Lo)}^2 = \frob{\pperpu(\Ltt[T + 1])}^2 + \sum\limits_{i \in \St{T}} \norm{\pperpu(\Lo_{i})}^2
		\leq \frob{\pperpu(\Mt[T + 1])}^2 + \sum\limits_{i \in \St{T}} \frac{33^2}{32^2} \mu^2 \frac{r}{n} \norm{\pperpu(\Ltt[T + 1])}^2\\
		\leq \frob{\pperpu(\Mt[T + 1])}^2 + 2\nThresh n \lprp{\frac{33}{32}}^2 \mu^2 \frac{r}{n} \frob{\pperpu(\Mt[T + 1])}^2\leq \frac{\epsilon}{10n},
	\end{multline*}
	where the first inequality follows from Lemma~\ref{lem:incPres} and using $\Mt[T+1]=\Ltt[T+1]+\Ctt[T+1]$, and the fact that $\Ltt[T+1]$ and $\Ctt[T+1]$ have different support. The second inequality follows from the fact that at most $2 \rho\cdot n$ points can be thresholded and then using \eqref{eq:cl3}. 
\end{proof}

\subsection{Arbitrary Noise---Theorem~\ref{thm:thmN}}
\label{sec:anGen1}
We now briefly discuss the proof of Theorem~\ref{thm:thmN}. In fact, we prove a stronger result: 
 \begin{theorem}
 	\label{thm:masterN}
 	Let $\Mo = \Lo + \Co + \No$ such that $\Lo$ satisfies Assumption~\ref{as:rank} and $\Co$ has column sparsity $\alpha \leq \frac{1}{128\mu^2 r}$. Furthermore, suppose that $\frob{\No} \leq \frac{\sigma^*_k}{16}$ for some $k \in [r]$. Then, Algorithm \ref{alg:ncpcacl} run with $\nThresh = \frac{1}{128\mu^2 r}$ and $\thresh$ set to $2\mu \sqrt{\frac{r}{n}}$ with $T=\log \frac{20\|\Mo\|_2\cdot n}{\epsilon}$, returns a subspace $U$ such that:
 	\begin{equation*}
 		\frob{(I - UU^\top) \Lo} \leq  3\frob{(I - U^*_k (U^*_k)^\top) \Lo} + 9 \frob{\No} + \frac{\epsilon}{10n}.
 	\end{equation*}
 \end{theorem}
Intuitively, the proof of Theorem~\ref{thm:masterN} proceeds along the same lines as that of Theorem~\ref{thm:thmCln} but requires significantly more careful analysis due to presence of noise and due to the outer loop. For example, due to the presence of noise, we cannot guarantee that Lemma~\ref{lem:clExpCl}, that was critical to proof of Theorem~\ref{thm:thmCln}, holds for all columns $i$. We show instead that the number of data points which have a large influence on the top-$k$ singular vectors is bounded (see Lemma~\ref{lem:clExpGen}). This ensures that the algorithm at least reaches the $k^{th}$ stage of the outer iteration before terminating. Similarly, we generalize Lemma~\ref{lem:clResCl} to handle $\No$ (see Lemma~\ref{lem:clResGen}). Finally, we present the key lemma that shows that if the algorithm does not terminate in the $k^{th}$ outer iteration, then it would have obtained a good approximation to the top-$k$ principal subspace of $\Lo$. 
\begin{lemma}
  \label{lem:inItPerfGen}
  Asume the conditions of Theorem~\ref{thm:thmN}. Furthermore, assume that Algorithm \ref{alg:ncpca} has not terminated during the $k^{th}$ outer iteration. Then, the iterate $U$ at the end of the $k^{th}$ outer iteration satisfies: 
  \begin{equation*}
  	\frob{(I - UU^\top) \Lo} \leq 3\frob{(I - \Uo_{1:k} (\Uo_{1:k})^\top) \Lo} + 9 \frob{\No} + \frac{\epsilon}{10n},
  \end{equation*}
  when Algorithm \ref{alg:ncpca} has been run with parameters $\nThresh = \frac{1}{128 \mu^2 r}$ and $\thresh = 2 \mu \sqrt{\frac{r}{n}}$.
\end{lemma}
See Appendix \ref{sec:pfInItPerfGen} for a detailed proof. We can now prove Theorem \ref{thm:masterN} as follows:
\begin{proof}
	Note that by Lemma \ref{lem:clExpGen}, the algorithm does not terminate before the completion of $k^{th}$ outer iteration. Now, suppose that the algorithm terminates at some iteration $k^\prime > k$. Then, by Lemma \ref{lem:inItPerfGen}, we have:

	\begin{equation*}
		\norm{\pperpu(\Lo)} \leq 3\frob{P_{\perp}^{\Uo_{1:k^\prime - 1}} (\Lo)} + 9 \frob{\No} + \frac{\epsilon}{10n} 
		\leq 3\frob{P_{\perp}^{\Uo_{1:k}} (\Lo)} + 9 \frob{\No} + \frac{\epsilon}{10n}.
	\end{equation*}
	This concludes the proof of the Theorem.
\end{proof}
\subsection{Gaussian Noise---Theorem~\ref{thm:thmG}}

Our analysis of \ncpcag\ show that the algorithm maintains the following critical invariant with high probability: 
\begin{invariant}
  \label{inv:incMnt}
  We assume that the following hold with respect to the two thresholding steps used in Algorithm~\ref{alg:ncpcag}. 
  \begin{enumerate}[leftmargin=*,labelindent=\labelwidth+\labelsep]
  	\item With respect to $\zeta_1$: If a column $i \not\in \supp{\Co}$ is thresholded, then the following condition holds:
  	\begin{equation*}
  		\norm{\No_{i}} \geq \sigma \lprp{\sqrt{d} + 2d^{\frac{1}{4}} \sqrt{\log\lprp{\frac{\mu^2 r}{c_2}}}}.
  	\end{equation*}
    and consequently only $\frac{3nc_2}{2\mu^2 r}$ points are removed in this step.
  	\item With respect to $\zeta_2$: If a thresholding step occurs due to the second thresholding step with $\zeta_2$, then at least half the points thresholded in this step are corrupted points. 
  \end{enumerate}
\end{invariant}

\begin{lemma}
	\label{lem:invGau}
	Assume the conditions of Theorem~\ref{thm:thmG}. Then, Invariant \ref{inv:incMnt} holds at any point in the running of Algorithm \ref{alg:ncpcag} with probability at least $1-\delta$.
\end{lemma}
	See Appendix \ref{sec:pfInvGau} for a detailed proof. 

Our proof then uses the above invariant along with a careful analysis of each of the two thresholding steps (Line 10, 12) to obtain the desired result. See Appendix \ref{sec:pfThmGauss} for a detailed proof. 


\section{Conclusions and Future Works}\label{sec:conc}
In this paper, we studied the outlier robust PCA problem. We proposed a novel thresholding based approach that, under standard regularity conditions, can accurately recover the top principal directions of the clean data points, as long as the number of outliers is less than $O(1/r)$ which is information theoretically tight up to constant factors. For noiseless or arbitrary noise case, our algorithms are based on two thresholding operators to detect outliers and leads to better recovery compared to existing methods in essentially the same time as that taken by vanilla PCA. For Gaussian noise, we obtain improved recovery guarantees but at a cost of higher run time.

Though our bounds have significant improvement over existing ones, they are still weaker than guarantees obtained by vanilla PCA (with out outliers). For instance, for arbitrary  noise, our errors are bounded in the Frobenius norm. In contrast, in absence of outliers, SVD can estimate the principal directions in operator norm. A challenging and important open problem is if the principal directions can be estimated in operator norm even in the presence  of outliers.

Similarly, for Gaussian noise, where each entry has variance $\sigma^2$, our result obtains an error bound of $O(\sigma \sqrt{n})$ which is significantly better than the Frobenius norm bound we get for arbitrary noise. But again in absence of outliers, SVD can estimate the principal directions exactly asymptotically. So, another open problem is if it is possible to do asymptotically consistent estimation of the principal directions with Gaussian noise in the presence of outliers. Moreover, our algorithm for the Gaussian setting is nearly a factor of $n$ slower than that for vanilla PCA. In order for this to be practical, it is very important to design an algorithm for this setting with nearly the same runtime as that of vanilla PCA.
\bibliography{refs}
\appendix

\section{Supplementary Results and Preliminaries}

Here, we will state and prove a few results useful in proving our main theorems. We will start by restating Weyl's perturbation inequality from \cite{bhatia}. 

\begin{theorem}
  \label{thm:wyl}
  Let $A \in \mathbb{R}^{d\times n}$. Furthermore, let $B = A + E$ for some matrix $E$. Then, we have that:
  \[
    \abs{\sigma(A)_i - \sigma(B)_i} \leq \norm{E} \quad \forall i \in min(d, n)
  \]
\end{theorem}




In the next lemma, we show that the singular values of the sum of two matrices with disjoint column supports are greater than either of the two matrices individually.

\begin{lemma}
  \label{lem:disSing}
  Let $A \in \mathbb{R}^{d \times n}$ and $B \in \mathbb{R}^{d \times n}$ be two matrices with disjoint column support. Then, we have $\forall i \in \min(d, n)$:
  \begin{equation*}
    \max (\sigma_i(A), \sigma_i(B)) \leq \sigma_i(A + B) 
  \end{equation*}
\end{lemma}

\begin{proof}
  Let the SVD of $A$ and $B$ be $U_A \Sigma_A V_A^\top$ and $U_B \Sigma_B V_B^\top$ respectively. The lemma holds for $i = 0$ as $\norm{v^\top (A + B)} \geq \max (\norm{v^\top A}, \norm{v^\top B})$. For any matrix $M$, $\sigma_i (M) = \min\limits_{U \in \mathbb{R}^{d \times (i - 1)}} \norm{(I - UU^\top) M}\ \forall i > 1$. For any $U$, there exist $v_1$ and $v_2$ in $Span((U_A)_{[i]})$ and $Span((U_B)_{[i]})$ respectively and $v_1^\top U = v_2^\top U = 0$. This is because the rank of $Span(U)$ is at most $(i - 1)$ and $Span((U_A)_{[i]})$ and $Span((U_B)_{[i]})$ are both rank-$i$ subspaces. Now, we have $\norm{v_1^\top (A + B)} \geq \norm{v_1^\top A} \geq \sigma_i{A}$ and $\norm{v_1^\top (A + B)} \geq \norm{v_1^\top B} \geq \sigma_i(B)$. The lemma by using either $v_1$ or $v_2$ for any $U$.
\end{proof}

The next lemma shows that an incoherent matrix remains incoherent even if a small number of columns have been set to $0$.

\begin{lemma}
  \label{lem:incPres}
  Let $L \in \mathbb{R}^{d \times n}$ be a $\mu$-column-incoherent, rank-$r$ matrix. Let $S \subset [n]$ such that $\abs{S} \leq \frac{1}{32\mu^2 r}$. Let $[U, S, V]$ and $[\US, \SigS, \VS]$ denote the SVDs of $L$ and $L_{\setminus S}$ respectively. Then, the following hold $\forall i \in [n]$:
  \begin{equation*}
    \text{Claim 1: }\norm{e_i^\top \VS} \leq \frac{33}{32} \mu \sqrt{\frac{r}{n}}\qquad\qquad \text{Claim 2: }\frac{31}{32} \sigma_i(L) \leq \sigma_i(L_{\setminus S}) \leq \sigma_i(L) 
  \end{equation*}
  Furthermore, each column $L_{i} \forall i \in [n]$ can be expressed as: 
  \begin{equation*}
    \text{Claim 3: } L_{i} = \US\SigS w_{i} \text{ with } \norm{w_i} \leq \frac{33}{32} \mu \sqrt{\frac{r}{n}}
  \end{equation*}
\end{lemma}
\begin{proof}
  Let $T$ be defined as the matrix $V$ with the rows in set $S$ set to $0$. We will first begin by proving that $T$ is full rank. Let $u \in \mathbb{R}^r$ and $\norm{u} = 1$:
  \begin{equation*}
    1 = \norm{Vu} \geq \norm{Tu} = \lprp{\sum\limits_{i = 1}^n \iprod{u}{V_{i, :}}^2 - \sum\limits_{j \in S} \iprod{u}{V_{j, :}}^2}^\frac{1}{2} \geq \lprp{1 - \sum\limits_{j \in S}\norm{V_{j, :}}^2}^{\frac{1}{2}} \geq \lprp{1 - \frac{1}{32}}^\frac{1}{2}
  \end{equation*}
  where the second inequality is obtained from the bound on $\abs{S}$ and $\norm{V_{j, :}}$. 
  Since $T$ and $\VS$ have the same column space, there exists a matrix $R \in \mathbb{R}^{r \times r}$ such that $TR = \VS$. We know that $R$ is full rank. We will now prove bounds on the singular values of $R$. For any $u \in \mathbb{R}^{r}$ and $\norm{u} = 1$
  \begin{equation*}
    \norm{Ru} = \norm{VRu} \geq \norm{TRu} = \norm{\VS u} = 1 = \norm{\VS u} = \norm{TRu} \geq \lprp{1 - \frac{1}{32}}^{\frac{1}{2}} \norm{Ru}
  \end{equation*}
  From this, we obtain the following inequality:
  \begin{equation*}
    1 \leq \norm{Ru} \leq \lprp{1 - \frac{1}{32}}^{-\frac{1}{2}}
  \end{equation*}
  From this, we have the first claim of the lemma as $TR = \VS$. We also know that $\US\SigS\VS^\top = U\Sigma T^\top$. Writing $\VS$ as $TR$, we have $\US \SigS R^\top T^\top = U\Sigma T^\top$. Using the fact that $T^\top$ is full rank, we have $\US \SigS R^\top = U\Sigma$. From this we have that $L = U\Sigma V^\top = \US \SigS R^\top V^\top$. Choosing $w_i = (R^\top V^\top)_{i}$, the second claim of the lemma follows. 

  For the final claim of the lemma, note that the singular values of $L_{\setminus S}$ are the same as the singular values of $U \Sigma (R^\top)^{-1}$. We know that $\sigma_{k + 1}(L_{\setminus S}) = \min\limits_{Q \in \mathbb{R}^{d \times k}} \norm{(I - QQ^\top) U\Sigma (R^\top)^{-1}}$. The upper bound follows from setting $Q$ to be the first $k$ singular vectors of $L$ and our bound on the singular values of $R$. For the lower bound, consider any $Q\in \mathbb{R}^{d \times k}$. $Span(Q)$ is a subspace of rank at most $k$. Therefore, there exists $v \in Span(U_{1})$ such that $\norm{v} = 1$ and $v^\top Q = 0$. We now have 
  \begin{equation*}
    \norm{v^\top (I - QQ^\top) U\Sigma (R^\top)^{-1}} = \norm{v^\top U\Sigma (R^\top)^{-1}} \geq \frac{\sigma_{k + 1}(L)}{\norm{R}} \geq \frac{31}{32} \sigma_{k + 1}(M)
  \end{equation*}
  Where the last inequality follows from our bounds on the singular values of $R$ and noting that the singular values of $R^{-1}$ are the inverses of the singular values of $R$. This proves the third claim of the lemma.
\end{proof}

We begin by stating a lemma used for bounding the length of Gaussian random vectors from \cite{laurent2000adaptive}:
\\
\begin{lemma}
\label{lem:gauConc}
  Let $Y_1, Y_2, \cdots , Y_d$  be i.i.d Gaussian random variables with mean $0$ and variance $1$. Let $Z = \sum\limits_{i = 1}^{d} \lprp{Y_i^2 - 1}$. Then the following inequality holds for any positive $x$:
  \[
    \mathbb{P} \lprp{Z \geq 2\sqrt{dx} + 2x} \leq \exp(-x)
  \]
\end{lemma}

We will now state the famous Bernstein's Inequality from \cite{boucheron2013concentration}.

\begin{theorem}
  \label{thm:bernstein}
  Let $X_1,\dots,X_n$ be independent real-valued random variables. Assume that there exist positive real numbers $\nu$ and $c$ such that $\sum\limits_{i = 1}^{n} \mathbb{E}\left[X_i^2\right] \leq \nu$ and
  \[
    \sum\limits_{i = 1}^{n} \mathbb{E}\left[(X_i)_{+}^q\right] \leq \frac{q!}{2}\nu c^{q - 2}\, \forall \, q \geq 3,
  \]
  where $x_{+} = \max (x, 0)$.

  If $S = \sum\limits_{i = 1}^{n} \lprp{X_i - \mathbb{E}[X_i]}$, then $\forall t \geq 0$, we have:
  \[
    \mathcal{P} \lprp{S \geq \sqrt{2\nu t} + ct} \leq \exp(-t)
  \]
\end{theorem}

We will now restate a lemma for controlling the singular values of a matrix with Gaussian random entries from \cite{DBLP:journals/corr/abs-1011-3027}.

\begin{lemma}
  \label{lem:gauRndNrmBnd}
  Let $A \in \mathbb{R}^{d \times n}$ be a random matrix whose entries are independent standard normal random variables. Then, for every $t \geq 0$, with probability at least $1 - 2\exp\lprp{-t^2 / 2}$, we have:
  \[
    \sqrt{n} - \sqrt{d} - t \leq \sigma_{min}(A) \leq \sigma_{max}(A) \leq \sqrt{n} + \sqrt{d} + t
  \]
\end{lemma}

\begin{corollary}
  \label{cor:probNrmBnd}
  Let $A \in \mathbb{R}^{d \times n}$ be a random matrix whose entries are independent standard normal random variables. For $n \geq 200 \lprp{d + 2\log\lprp{\frac{2}{\delta}}}$, we have:
  \[
    0.9 \sqrt{n} \leq \sigma_{min}(A) \leq \sigma_{max}(A) \leq 1.1 \sqrt{n}
  \]
  with probability at least $1 - \delta$
\end{corollary}

\begin{lemma}
  \label{lem:gauLthConc}
  Let $Y_1, Y_2, \dots, Y_n$ be iid $d$-dimensional random vectors such that $Y_i \sim \mathcal{N}\lprp{0, I} \forall i \in [n]$. Then, we have for any $c_2 \leq 1$:
  \begin{equation*}
    \mathcal{P}\lprp{\abs{\left\{i: \norm{Y_i} \geq d^{\frac{1}{2}} + 2d^{\frac{1}{4}} \lprp{\log\lprp{\frac{1}{c_2}} + \log\lprp{\mu^2r}}^{\frac{1}{2}} \right\}} \geq \frac{3c_2n}{2\mu^2r}} \leq \beta
  \end{equation*}
  when $n \geq \frac{16\mu^2r}{c_2} \log\lprp{\frac{1}{\beta}}$.
\end{lemma}

\begin{proof}
  Let $Y_1,\dots,Y_n$ be iid random vectors such that $Y_i \sim \mathcal{N}\lprp{0, I} \forall i \in [n]$. From Lemma~\ref{lem:gauConc}, we have that:
  \[
    \mathcal{P}\lprp{\norm{Y_i} \geq d^{1/2} + 2d^{\frac{1}{4}}\lprp{\log\lprp{\frac{1}{c_2}} + \log\lprp{\mu^2r}}^{1 / 2}} \leq \frac{c_2}{\mu^2r}
  \]
  Let $p \coloneqq \mathcal{P}\lprp{\norm{Y_i} \geq d^{1/2} + 2d^{\frac{1}{4}}\lprp{\log\lprp{\frac{1}{c_2}} + \log\lprp{\mu^2r}}^{1 / 2}}$. Consider random variables $Z_i \forall i \in [n]$ be defined such that $Z_i = \mathbb{I}\left[\norm{Y_i} \geq d^{1/2} + 2d^{\frac{1}{4}}\lprp{\log\lprp{\frac{1}{c_2}} + \log\lprp{\mu^2r}}^{1 / 2}\right]$. Note that $Z_i$ satisfy the conditions of Theorem~\ref{thm:bernstein} with $\nu = np$ and $c = 1$. We can now bound the probability that $\sum\limits_{i = 1}^{n} Z_i$ is large by setting $t = \frac{nc_2}{16\mu^2r}$:
  \[
    \mathcal{P}\lprp{\sum\limits_{i = 1}^{n} Z_i \leq \frac{3nc_2}{2\mu^2r}} \leq \mathcal{P}\lprp{\sum\limits_{i = 1}^{n} Z_i \leq \sum\limits_{i = 1}^{n} \mathbb{E}\left[Z_i\right] + \frac{nc_2}{2\mu^2r}} \leq \mathcal{P}\lprp{\sum\limits_{i = 1}^{n} Z_i \leq np + \sqrt{2npt} + t} \leq \exp\lprp{-t}
  \]
  For our choice of $n$, the above inequality implies the lemma.
\end{proof}

\section{Proof of Technical Lemmas}
\subsection{Proof of Lemma \ref{lem:clResCl}}\label{sec:pfclResCl}
\begin{proof}
	We prove the lemma through a series of inequalities:
	\begin{multline*}
	\norm{(I - \US \US ^\top) \Lo_{i}} \overset{(\zeta_1)}{\leq} \frac{33}{32} \mu \sqrt{\frac{r}{n}} \norm{(I - \US\US^\top) \LoS} \overset{(\zeta_2)}{\leq} \frac{33}{32} \mu \sqrt{\frac{r}{n}} \norm{(I - \US(\US)^\top) \MoS} \\ 
	\leq \frac{33}{32} \mu \sqrt{\frac{r}{n}} \norm{(I - \Uo(\Uo)^\top) \MoS},
	\end{multline*}
	where $(\zeta_1)$ holds from Lemma \ref{lem:incPres} and $(\zeta_2)$ follows by using the fact that $\LoS$ can be obtained by setting a few columns of $\MoS$ to $0$. The last inequality follows from the fact that $\US$ contains the top-$r$ singular vectors of $\MoS$.
\end{proof}

\subsection{Proof of Lemma \ref{lem:clExpCl}}
\label{sec:pfclExpCl}
\begin{proof}
	The lemma can be proved through the following set of inequalities:
	\begin{multline*}
	\norm{\SigS^{-1} \US^\top \Lo_{i}} \overset{(\zeta_1)}{\leq} \norm{\SigS^{-1} \US^\top \UoS \SigoS w} \leq \frac{33}{32} \mu \sqrt{\frac{r}{n}} \norm{\SigS^{-1} \US^\top \LoS} \\ 
	\overset{(\zeta_2)}{\leq} \frac{33}{32} \mu \sqrt{\frac{r}{n}} \norm{\SigS^{-1} \US^\top \MoS} \leq \frac{33}{32} \mu \sqrt{\frac{r}{n}},
	\end{multline*}
	where $(\zeta_1)$ holds with $\norm{w} \leq \frac{33}{32} \mu \sqrt{\frac{r}{n}}$ from Lemma \ref{lem:incPres} and $(\zeta_2)$ follows from the fact that $\LoS$ can be obtained from $\MoS$ by setting some columns in $\MoS$ to $0$.
\end{proof}

\subsection{Lemma \ref{lem:clExpGen}}
\label{sec:pfClExpGen}

\begin{lemma}
	\label{lem:clExpGen}
	Consider the setting of Theorem~\ref{thm:thmN}. Let $S \subset [n]$ denote any subset of the columns of $\Mo$ such that $\abs{S} \leq 3\nThresh n$. Furthermore, suppose that $\frob{\No} \leq \frac{\sigma_k (\Lo)}{16}$ for some $k \in [r]$. Let $\MoS$($\LoS$, $\NoS$, $\CoS$) denote the matrix $\Mo$($\Lo$, $\No$, $\Co$) projected onto the columns not in $S$. Let $\US\SigS \VS^\top$ denote the rank-$k^\prime$ SVD of $\MoS$ for some $k^\prime \leq k$. Then, we have:
	\[
	\# \lprp{i: \norm{E_{i}} \geq 2\mu \sqrt{\frac{r}{n}}} \leq 2 \nThresh n,
	\]
	where $E = \SigS^{-1} \US^\top \Mo$
\end{lemma}

\begin{proof}
  From Lemma~\ref{lem:incPres}, we get that $\sigma_{k^\prime}(\LoS) \geq \frac{31}{32} \sigma_{k^\prime}(\Lo)$. Along with Theorem~\ref{thm:wyl}, we conclude that $\sigma_{k^\prime}(\LoS + \NoS) \geq \frac{7}{8} \sigma_{k^\prime}(\Lo)$. Since the column supports of $\Lo + \No$ and $\Co$ are disjoint, we have that $\sigma_{k^\prime}(\LoS + \NoS) \leq \sigma_{k^\prime}(\MoS)$. That is, 
  \begin{equation}\label{eq:clEG1}
  \sigma_{k^\prime}(\MoS)\geq \frac{7}{8} \sigma_{k^\prime}(\Lo). 
  \end{equation}
  
  We first bound the quantity $\norm{\SigS^{-1} \US^\top \Lo_{i}} \forall i \in [n]$:  
  \begin{multline}\label{eq:clEG2}
    \norm{\SigS^{-1} \US^\top \Lo_{i}} \overset{(\zeta_1)}{\leq} \frac{33}{32}\mu \sqrt{\frac{r}{n}} \norm{\SigS^{-1} \US^\top \LoS} \overset{(\zeta_2)}{\leq} \frac{33}{32}\mu \sqrt{\frac{r}{n}} \lprp{\norm{\SigS^{-1} \US^\top (\LoS + \NoS)} + \norm{\SigS^{-1} \US^\top \NoS}} \\ 
    \overset{(\zeta_3)}{\leq} \frac{33}{32}\mu \sqrt{\frac{r}{n}} \lprp{\norm{\SigS^{-1} \US^\top \MoS} + \norm{\SigS^{-1} \US^\top \NoS}} \overset{(\zeta_4)}{\leq} \frac{33}{32}\mu \sqrt{\frac{r}{n}} \lprp{1 + \frac{1}{14}} \leq \frac{9}{8} \mu \sqrt{\frac{r}{n}},
  \end{multline}
where $\zeta_1$ follows from Lemma~\ref{lem:incPres}, $\zeta_2$ using triangle inequality, $\zeta_3$ using the above given fact that ${\cal SVD}_r(\MoS)=\US\SigS(\VS)^\top$, $\zeta_4$ follows from using \eqref{eq:clEG1} with bound on $\|\No\|_F$. 
 
  Suppose $\norm{\SigS^{-1} \US^\top (\Lo_{i} + \No_{i})} \geq 2 \mu \sqrt{\frac{r}{n}}$ for some $i$. We now have: 
  \[
    \norm{\SigS^{-1} \US^\top \No_{i}} \geq \frac{7}{8} \mu \sqrt{\frac{r}{n}}.
  \]
Similarly, using \eqref{eq:clEG1}, we get that: 
  \[
    \norm{\No_{i}} \geq \frac{3}{4} \mu \sqrt{\frac{r}{n}} \sigma_{k^\prime}(\Lo).
  \]

  Let $\Gamma \coloneqq \{i: \norm{\SigS^{-1} \US^\top (\Lo_{i} + \No_{i})} \geq 2 \mu \sqrt{\frac{r}{n}}\}$. Let $\No_\Gamma$ denote the matrix $\No$ restricted to the set $\Gamma$. Then we have,
  \[
    \sqrt{|\Gamma|} \frac{3}{4} \mu \sqrt{\frac{r}{n}} \sigma_{k^\prime}(\Lo) \leq \norm{\No_\Gamma} \leq \norm{\No} \leq \frac{1}{16}\sigma_{k^\prime}(\Lo).
  \]
  This implies that $|\Gamma| \leq \frac{n}{144\mu^2r} \leq \nThresh n$. Also, by our assumption, $\colSp \leq \nThresh$, i.e., number of non-zero $\Co_i$ is less than $\rho n$. That is, the set $\{i: \norm{\SigS^{-1} \US^\top \Co_{i}} \geq 2 \mu \sqrt{\frac{r}{n}}\}$ is of size at most $\nThresh n$. Using the fact that support of $\Co$ and $\Lo+\No$ is disjoint, we have that the set $\{i: \norm{\SigS^{-1} \US^\top \Mo_{i}} \geq 2 \mu \sqrt{\frac{r}{n}}\}$ is of size at most $2\nThresh n$.
\end{proof}

\subsection{Proof of Lemma \ref{lem:clResGen}}
\label{sec:pfClResGen}

\begin{lemma}
	\label{lem:clResGen}
	Assume the setting of Lemma~\ref{lem:clExpGen}. Let $\US \SigS \VS^\top$ ($\UoS \SigoS (\VoS)^\top$) be the rank-$k\ \mathcal{SVD}$ of $\MoS$ ($\LoS$), then the following holds $\forall i, 1\leq i\leq n$:
	\begin{equation*}
	\norm{(I - \US\US^\top) (\Lo_{i} + \No_{i})} \leq \frac{33}{32} \mu \sqrt{\frac{r}{n}} \lprp{\norm{(I - \Uo_{1:k}(\Uo_{1:k})^\top) \MoS} + \norm{\No}} + \norm{\No_{i}}.
	\end{equation*}
\end{lemma}

\begin{proof}
	\begin{align*}
  \label{eqn:clRes}
    \norm{(I - \US\US^\top) (\Lo_{i} + \No_{i})} &\leq \norm{\No_{i}} + \norm{(I - \US\US^\top) \Lo_{i}} \overset{(\zeta_1)}{\leq} \norm{\No_{i}} + \frac{33}{32} \mu\sqrt{\frac{r}{n}} \norm{(I - \US\US^\top) \LoS} \\ 
    &\leq \norm{\No_{i}} + \frac{33}{32} \mu\sqrt{\frac{r}{n}} \lprp{\norm{(I - \US\US^\top) (\LoS + \NoS)} + \norm{\NoS}} \\
    &\leq \norm{\No_{i}} + \frac{33}{32} \mu\sqrt{\frac{r}{n}} \lprp{\norm{(I - \US\US^\top) \MoS} + \norm{\No}} \\
    &\leq \norm{\No_{i}} + \frac{33}{32} \mu\sqrt{\frac{r}{n}} \lprp{\norm{(I - \Uo_{1:k}(\Uo_{1:k})^\top) \MoS} + \norm{\No}}
  \end{align*}
  where $(\zeta_1)$ follows from Lemma~\ref{lem:incPres} and the fact that only $3\nThresh n$ columns are ever thresholded at any stage of the algorithm. The remaining inequalities follow using triangle inequality and $\MoS=\LoS+\NoS$ along with ${\cal SVD}_r(\MoS)=\US \SigS (\VS)^\top$. 
\end{proof}

\subsection{Proof of Lemma \ref{lem:inItPerfGen}}
\label{sec:pfInItPerfGen}
\begin{proof}
  Let $\St{t}$ denote the columns of $\Co$ that are not thresholded in the $t^{th}$ inner iteration. For each $j \in \St{t}$, we know that $\norm{(\Sigot{t})^{-1} (\Uot{t})^\top \Co_{j}} \leq 2\mu \sqrt{\frac{r}{n}}$ from our assumption on the termination of the algorithm. Furthermore, since $\Co_{j}$ is not thresholded, we can associate a unique column $i_j$ which is thresholded and $i_j \not\in \supp{\Co}$ such that $\norm{I - \Uot{t}(\Uot{t})^\top \Mo_{i_j}} \geq \norm{I - \Uot{t}(\Uot{t})^\top \Mo_{j}}$. Let $y^{i, t} \coloneqq (\Uot{t})^{-1} (\Sigot{t})^{-1} \Mo_{i}$ and $r^{i, t} \coloneqq (I - \Uot{t}(\Uot{t})^\top) \Mo_{i},  \forall i$. Thus we have: 
  \begin{equation*}
    \norm{y^{j, t}} \leq 2\mu\sqrt{\frac{r}{n}}, \qquad \norm{r^{j, t}} \leq \norm{r^{i_j, t}}. 
  \end{equation*}

  Let $\Qt{t}$ denote the columns of $\Lo$ that have been thresholded in the $t^{th}$ iteration. Furthermore, we definite the matrices $\Ltt[t + 1] \coloneqq \Lo_{\setminus \Qt{t}}$, $\Ntt[t + 1] \coloneqq \No_{\setminus \Qt{t}}$ and $\Ctt[t + 1] \coloneqq \Co_{\St{t}}$. Recall the notation, $\pperpu(M)=(I-UU^T)M$. We now have for any $t \geq 0$:
  \begin{align}
    & \frob{\pperputo(\Lo)} = \frob{\pperputo(\Ltt[t + 1] + (\Lo - \Ltt[t + 1]))} \leq \frob{\pperputo(\Ltt[t + 1])} + \frob{\pperputo(\Lo - \Ltt[t + 1])} \notag\\
    & \leq \frob{\pperputo(\Ltt[t + 1])} + \lprp{\sum\limits_{i \in \Qt{t}} \norm{\pperputo(\Lo_i)}^2}^{\frac{1}{2}} \overset{(\zeta_1)}{\leq} \frob{\pperputo(\Ltt[t + 1])}\left(1 + \sqrt{3\nThresh n} \frac{33}{32} \mu \sqrt{\frac{r}{n}}\right)\notag\\
    & \overset{(\zeta_2)}{\leq} \frac{5}{4} \lprp{\frob{\pperputo(\Ltt[t + 1] + \Ntt[t + 1])} + \frob{\pperputo(\Ntt[t + 1])}} \notag\\
    & \overset{(\zeta_3)}{\leq} \frac{5}{4} \lprp{\frob{\pperputo(\Ltt[t + 1] + \Ntt[t + 1] + \Ctt[t + 1])} + \frob{\Ntt[t + 1]}}\notag \\
    & \overset{(\zeta_4)}{\leq} \frac{5}{4} \lprp{\pperpuok (\Ltt[t + 1] + \Ntt[t + 1] + \Ctt[t + 1])} + \frob{\Ntt[t + 1]}  \overset{(\zeta_5)}{\leq} \frac{5}{4} \lprp{\frob{\pperpuok(\Mt[t + 1])} + \frob{\No}},\label{eq:ipg0}
  \end{align}
  where $(\zeta_1)$ follows from Lemma~\ref{lem:incPres}, $(\zeta_2)$ from triangle inequality and bound over $\rho$, $(\zeta_3)$ from Lemma~\ref{lem:disSing}, $(\zeta_4)$ from the properties of the SVD and $(\zeta_5)$ from Lemma~\ref{lem:disSing}.

  We will now show that $\frob{\pperpuok(\Mt[t + 1])}$ decreases at a geometric rate:
  \begin{align}
    & \frob{\pperpuok( \Mt[t + 1]}  \overset{(\zeta_6)}{\leq} \frob{\pperpuok(\Ltt[t + 1])} + \frob{\No} + \lprp{\frob{\pperpuok( \Ctt[t + 1])}} \notag\\
    & \overset{(\zeta_7)}{\leq} \frob{\pperpuok( \Lo)} + \frob{\No} + \lprp{\sum\limits_{j \in \St{t - 1}}\norm{\pperpuok( (\Ctt[t + 1])_{j})}^2}\notag \\
    & \overset{(\zeta_8)}{\leq} \frob{\pperpuok( \Lo)} + \frob{\No} + \lprp{2\sum\limits_{j \in \St{t}} \norm{\pperpuok( \Uot{t}\Sigot{t} y^{j, t}}^2 + \norm{\pperpuok( r^{j, t})}^2}^{\frac{1}{2}} \notag\\
    & \leq \frob{\pperpuok( \Lo)} + \frob{\No} + \lprp{8\nThresh n \frac{\mu^2 r}{n} \norm{\pperpuok( \Uot{t}\Sigot{t})}^2 + 2\sum\limits_{j \in \St{t}} \norm{\pperpuok( r^{j, t})}^2}^{\frac{1}{2}} \notag\\
    & \overset{(\zeta_9)}{\leq} \frob{\pperpuok( \Lo)} + \frob{\No} +\lprp{\frac{1}{8} \norm{\pperpuok( \Uot{t}\Sigot{t})}^2 + 2\sum\limits_{j \in \St{t}} \norm{r^{i_j, t}}^2}^{\frac{1}{2}} \notag\\
    & \overset{(\zeta_{10})}{\leq} \frob{\pperpuok( \Lo)} + \frob{\No} + \left(\frac{1}{8} \underbrace{\norm{\pperpuok( \Uot{t}\Sigot{t})}^2}_{Term\ 1} \vphantom{+ 2 \underbrace{\sum\limits_{j \in \St{t - 1}} \norm{\No_{i_j}}^2 + \lprp{\frac{33}{32}}^2 \mu^2\frac{r}{n} \lprp{\norm{(I - \Uot{t - 1}(\Uot{t - 1})^\top) (\Ltt[t - 1] + \Ntt[t - 1])}}^2}_{Term\ 2}} \right. \notag\\ 
    & \qquad\qquad\qquad\qquad \left. + 4 \underbrace{\sum\limits_{j \in \St{t}} \norm{\No_{i_j}}^2 + \lprp{\frac{33}{32}}^2 \mu^2\frac{r}{n} \lprp{\norm{\pperput(\Ltt[t] + \Ntt[t])}}^2}_{Term\ 2}\right)^{\frac{1}{2}},\label{eq:ipg1}
  \end{align}
  where $(\zeta_6)$ follows from triangle inequality, $(\zeta_7)$ follows from Lemma~\ref{lem:disSing}, $(\zeta_8)$ follows from triangle inequality and the fact that $(a + b)^2 \leq 2a^2 + 2b^2$, $(\zeta_9)$ from our previous observations about $y^{j, r}$ and $r^{j, t}$ and $(\zeta_10)$ from Lemma~\ref{lem:clExpGen}.

  We will now proceed to bound Term 1 as follows:
  \begin{equation}\label{eq:ipg2}
    \frob{\pperpuok( \Uot{t}\Sigot{t})}^2 \overset{(\zeta_{11})}{\leq} \frob{\pperpuok( \Mt[t])}^2,
  \end{equation}
  where $(\zeta_{11})$ follows from considering the full SVD of $\Mt[t]$ and Lemma~\ref{lem:disSing}.

  We now proceed to bound Term 2 as:
  \begin{multline}\label{eq:ipg3}
    \sum\limits_{j \in \St{t}} \norm{\No_{i_j}}^2 + \lprp{\frac{33}{32}}^2 \mu^2\frac{r}{n} \lprp{\norm{\pperput(\Ltt[t] + \Ntt[t])}}^2 \leq \frob{\No}^2 + \frac{9}{8} \nThresh n \frac{\mu^2 r}{n} \frob{\pperput(\Mt[t])}^2 \\
    \leq \frob{\No}^2 + \frac{1}{32} \frob{\pperpuok( \Mt[t])}^2,
  \end{multline}
  where the first inequality follows from Lemma~\ref{lem:disSing} and the second inequality from the fact that $\Uot{t+1}$ are top-$k$ left singular vectors of $\Mt[t]$.  

  Using \eqref{eq:ipg1}, \eqref{eq:ipg2}, \eqref{eq:ipg3}, we have: 
  \begin{align*}
    \frob{\pperpuok( \Mt[t + 1])}
     &\leq \frob{\pperpuok( \Lo)} + \frob{\No} + \left(\frac{1}{8}\frob{\pperpuok( \Mt[t])}^2 + 4\frob{\No}^2 + \frac{1}{8}\frob{\pperpuok( \Mt[t])}^2 \right)^\frac{1}{2} \\
    & \leq \frob{\pperpuok( \Lo)} + 3\frob{\No} + \frac{1}{2}\frob{\pperpuok( \Mt[t])},
  \end{align*}
  where the first inequality follows from Lemma~\ref{lem:disSing} and considering the full SVD of $\Mt[t]$ and the last inequality follows from the fact that $\sqrt{a + b + c} \leq \sqrt{a} + \sqrt{b} + \sqrt{c}$.

  By recursively applying the above inequality, we have:
  \begin{equation*}
   \frob{\pperpuok( \Mt[T + 1])} \leq 2\frob{\pperpuok( \Lo)} + 6\frob{\No} + \frac{\epsilon}{20n}.
  \end{equation*}

Lemma now follows using \eqref{eq:ipg0} with the above equation. 
\end{proof}

\subsection{Lemma \ref{lem:gauipConc}}
\label{sec:pfGauIpConc}

\begin{lemma}
	\label{lem:gauipConc}
	Let $Y_1, Y_2, \dots, Y_n$ be iid $d$-dimensional random vectors such that $Y_i \sim \mathcal{N}\lprp{0, I} \forall i \in [n]$. Then, we have for any $c_1 \leq 1$:
	\begin{equation*}
	\mathcal{P}\lprp{\exists v\in \mathbb{R}^d,\ \norm{v} = 1\ s.t\ \abs{\left\{i: \iprod{v}{Y_i} \geq 2\lprp{\log\lprp{\mu^2r} + \log\lprp{d} + \log\lprp{\frac{1}{c_1}}}^{\frac{1}{2}}\right\}} \geq \frac{3c_1n}{\mu^2rd}} \leq \beta,
	\end{equation*}
	when $n \geq \frac{16\mu^2rd}{c_1} \left[\log\lprp{\frac{1}{\beta}} + d\log\lprp{80d}\right]$.
\end{lemma}

\begin{proof}
  Let $v \in \mathbb{R}^d \text{ s.t } \norm{v} = 1$. We define the set $\mathcal{S}_{v, \theta}$ as follows:
  \begin{equation*}
     \mathcal{S}_{v, \theta} = \{u: u \in \mathbb{R}^d \wedge \norm{u} = 1 \wedge \iprod{u}{v} \geq \cos(\theta)\}.
  \end{equation*}
  We now define the set $\mathcal{T}\lprp{v, \theta, \delta}$ as:
  \begin{equation*}
    \mathcal{T}\lprp{v, \theta, \delta} = \{x: x \in \mathbb{R}^d \wedge \exists u \in \mathcal{S}_{v, \theta} \quad s.t \quad \iprod{u}{x} \geq \delta\}.
  \end{equation*}
  Now, let $y \sim \mathcal{N}(0, I)$. Using spherical symmetry of the Gaussian, w.l.o.g. $v = e_1$. We now define the complementary sets $\mathcal{Q}\lprp{\nu}$ and $\mathcal{R}\lprp{\nu}$ as:
  \begin{equation*}
    \mathcal{Q}\lprp{\nu} = \{x: x\in \mathbb{R}^d \wedge x_1 < \nu\}, \qquad \mathcal{R}\lprp{\nu} = \{x: x\in \mathbb{R}^d \wedge x_1 \geq \nu\}.
  \end{equation*}

  We will now bound the probability that $y\in \mathcal{T}\lprp{v, \theta, \delta}$ for $\delta = 2\lprp{\log\lprp{\mu^2r} + \log\lprp{d} + \log\lprp{\frac{1}{c_1}}}^{1 / 2}$ and $\theta = \csc^{-1} \lprp{10(d - 1)^{1/2}}$.

  \begin{align*}
    \mathcal{P}\lprp{y \in \mathcal{T}\lprp{e_1, \theta, \delta}} &= \int\limits_{\mathcal{T}\lprp{e_1, \theta, \delta}} \frac{1}{\lprp{\sqrt{2\pi}}^d} \exp\lprp{-\frac{\norm{y}^2}{2}} dy \\&= \int\limits_{\mathcal{T}\lprp{e_1, \theta, \delta} \cap \mathcal{Q}\lprp{\delta / \sqrt{2}}} \frac{1}{\lprp{\sqrt{2\pi}}^d} \exp\lprp{-\frac{\norm{y}^2}{2}} dy + \int\limits_{\mathcal{T}\lprp{e_1, \theta, \delta} \cap \mathcal{R}\lprp{\delta / \sqrt{2}}} \frac{1}{\lprp{\sqrt{2\pi}}^d} \exp\lprp{-\frac{\norm{y}^2}{2}} dy \\
    &\leq \int\limits_{\mathcal{R}\lprp{\delta / \sqrt{2}}} \frac{1}{\lprp{\sqrt{2\pi}}^d} \exp\lprp{-\frac{\norm{y}^2}{2}} dy + \int\limits_{\mathcal{T}\lprp{e_1, \theta, \delta} \cap \mathcal{Q}\lprp{\delta / \sqrt{2}}} \frac{1}{\lprp{\sqrt{2\pi}}^d} \exp\lprp{-\frac{\norm{y}^2}{2}} dy \\
    &\overset{(\zeta_1)}{\leq} \frac{c_1}{\mu^2rd} + \int\limits_{\mathcal{T}\lprp{e_1, \theta, \delta} \cap \mathcal{Q}\lprp{\delta / \sqrt{2}}} \frac{1}{\lprp{\sqrt{2\pi}}^d} \exp\lprp{-\frac{\norm{y}^2}{2}} dy,
  \end{align*}
  where $(\zeta_1)$ follows from the fact that for $t \geq 1$, $\int\limits_{t}^{\infty} \frac{1}{\sqrt{2\pi}} \exp\lprp{- \frac{x^2}{2}} dx \leq \exp\lprp{-\frac{t^2}{2}}$.

  We will use $\mathcal{M}\lprp{\theta, \delta, \gamma}$ to denote the set $\{z: z \in \mathbb{R}^{d - 1} \wedge (\gamma, z) \in \mathcal{T}\lprp{e_1, \theta, \delta} \cap \mathcal{Q}\lprp{\delta / \sqrt{2}}\}$.
  We can bound the second term as follows:
  \begin{multline}
    \label{eqn:gauMedTrm2}
    \int\limits_{\mathcal{T}\lprp{e_1, \theta, \delta} \cap \mathcal{Q}\lprp{\delta / \sqrt{2}}} \frac{1}{\lprp{\sqrt{2\pi}}^d} \exp\lprp{-\frac{\norm{y}^2}{2}} dy \\ 
    \leq \int\limits_{-\infty}^{\delta / \sqrt{2}} \frac{1}{\sqrt{2\pi}}\exp\lprp{-\frac{y_1^2}{2}} \lprp{\int\limits_{\mathcal{M}\lprp{\theta, \delta, y_1}} \frac{1}{\lprp{\sqrt{2\pi}}^{d-1}} \exp\lprp{-\frac{\norm{z}^2}{2}} dz} dy_1.
  \end{multline}
  Now, let $z \in \mathbb{R}^{d}$ be such that $z_1 = y_1 \wedge z_{2:d} \in \mathcal{M}\lprp{\theta, \delta, y_1}$ for some $y_1 \in [-\infty, \delta / \sqrt{2}]$. Therefore, $\exists w \in \mathcal{S}_{v, \theta}$ such that $\iprod{w}{z} \geq \delta$. We can decompose $w$ into its components along $v$ and orthogonal to it, $w = \cos(\theta^\prime) v + \sin(\theta^\prime) v^\perp$ for some unit vector $v^\perp$ orthogonal to $v$ and some $\theta^\prime \in [0, \theta]$. We know that $\iprod{w}{z} \geq \delta$ and that $\iprod{w}{v} \leq \delta / \sqrt{2}$. From these two inequalities and using the fact that $v = e_1$, we get:
  \[
    \sin(\theta^\prime) \norm{z_{2:d}}\geq \sin(\theta^\prime) \iprod{v^\perp}{z} \geq \delta - \cos(\theta^\prime)\iprod{v}{z} \geq \delta - \cos(\theta^\prime) \frac{\delta}{\sqrt{2}} \geq \lprp{1 - \frac{1}{\sqrt{2}}} \delta.
  \]
  This allows us to lower bound the length of $z_{2:d}$ by $10(d-1)^{1/2}\lprp{1 - \frac{1}{\sqrt{2}}}\delta$. For our choice of $\delta$ and $\theta$ and using Lemma~\ref{lem:gauConc}, we now get that the inner integration in equation~\ref{eqn:gauMedTrm2} is atmost $\frac{c_1}{\mu^2rd}$. Thus, we have the following bound on $\mathcal{P}\lprp{y \in \mathcal{T}\lprp{e_1, \theta, \delta}}$:
  \begin{equation}
    \label{eqn:probBound}
    \mathcal{P}\lprp{y \in \mathcal{T}\lprp{e_1, \theta, \delta}} \leq \frac{2c_1}{\mu^2rd}.
  \end{equation}
  Let $p$ be used to denote the value $\mathcal{P}\lprp{y \in \mathcal{T}\lprp{e_1, \theta, \delta}}$. Now, assume $Y_1, \dots, Y_n$ are iid random vectors with $Y_i \thicksim \mathcal{N}\lprp{0, I} \ \forall i \in [n]$. Now let $Z_i$ be defined such that $Z_i = \mathbb{I}\left[Y_i \in \mathcal{T}\lprp{e_i, \theta, \delta}\right] \forall i \in [n]$. Note that $Z_i$ is a Bernoulli random variable which is $1$ with probability $p$. It can be seen that $Z_i$ satisfy satisfy the conditions of \ref{thm:bernstein} with $\nu = np$ and $c = 1$. Therefore, setting $t = \frac{nc_1}{16\mu^2rd}$ in Theorem~\ref{thm:bernstein}, we get:
  \begin{equation}
    \label{eqn:oneDirBound}
    \mathcal{P}\lprp{\sum\limits_{i = 1}^{n}Z_i \geq \frac{3nc_1}{\mu^2rd}} \leq \mathcal{P}\lprp{\sum\limits_{i = 1}^{n}Z_i \geq np + \frac{nc_1}{\mu^2rd}} \leq \mathcal{P}\lprp{\sum\limits_{i = 1}^{n}Z_i \geq np + \sqrt{2\nu t} + t} \leq \exp\lprp{-t}.
  \end{equation}
  Now, consider the subset $\mathcal{K} \coloneqq \{x: x\in \mathbb{R}^d \wedge \abs{x_i} \leq 1 \forall i \in [d]\}$. Consider a partitioning of $\mathcal{K}$ into subsets $\mathcal{K}\lprp{\epsilon, j} = \{x: x\in \mathcal{K} \wedge \forall i \in [d] j_i\epsilon - 1 \leq x_i \leq (j_i + 1)\epsilon - 1\}$ where $j \in \mathcal{J}$ is an index for each of these subsets. Note that for any $\epsilon$, at most $\lprp{\ceil{\frac{2}{\epsilon}}}^d$ such indices are required to ensure that $\mathcal{K} \subseteq \bigcup_{j \in \mathcal{J}} \mathcal{K}\lprp{\epsilon, j}$. Setting $\epsilon = \frac{1}{40d}$, we have for any two unit vectors $v_1$ and $v_2$ such that $v_1, v_2 \in \mathcal{K}\lprp{\epsilon, j}$ for some $j$, $\norm{v_1 - v_2} \leq \frac{1}{40d^{1 / 2}}$. From this fact, it can be seen that Equation~\ref{eqn:oneDirBound} holds for all unit vectors in $\mathcal{K}\lprp{\epsilon, j}$ with any unit vector $v \in \mathcal{K}\lprp{\epsilon, j}$. Therefore, we choose for each subset $\mathcal{K}\lprp{\epsilon, j}$, which contains a unit vector, a unit vector $v$ and take an union bound over all such subsets $\mathcal{K}\lprp{\epsilon, j}$. After doing so we get the following bound:
  \begin{equation}
    \mathcal{P}\lprp{\exists v \in \mathbb{R}^d \wedge \norm{v} = 1 \text{ s.t } \sum\limits_{i = 1}^{n} Z_i \geq \frac{3nc_1}{\mu^2rd}} \leq (80d)^d \exp\lprp{-t} \overset{(\zeta_2)}{\leq} \beta,
  \end{equation}
  where $(\zeta_2)$ follows from the conditions of the theorem. Thus, we have proved the theorem.
\end{proof}

\subsection{Lemma \ref{lem:clExpGau}}
\label{sec:pfClExpGau}

\begin{lemma}
	\label{lem:clExpGau}
	Assume the conditions of Theorem~\ref{thm:thmG}. Let $S \subset [n]$ denote any subset such that $\abs{S} \leq \frac{1}{64\mu^2 r}$. Let $\MoS$($\LoS$, $\NoS$, $\CoS$) denote the matrices $\Mo$ ($\Lo$, $\No$, $\Co$) restricted to the columns not in $S$. Let $[\US, \SigS, \VS] = \mathcal{SVD}_{r + 1}(\MoS)$. Furthermore, let $\mathcal{E} = \{x: x = \US\SigS y \text{ for some } \norm{y} \leq 2\mu \sqrt{r / n}\}$. Then $\forall i$, we have:
	\begin{equation*}
	\norm{\Pr{\US}\lprp{\Lo_{i} + \No_{i}} - \Pr{\mathcal{E}}\lprp{\Lo_{i} + \No_{i}}} \leq \frac{33}{32}\mu\sqrt{\frac{r}{n}}\norm{\No} + \norm{\Pr{\US} (\No_{i})},
	\end{equation*}
	where $\Pr{\US}\lprp{M}=\US\US^\top M$. 
\end{lemma}

\begin{proof}
  Let $[\UoS, \SigoS, \VoS] = \mathcal{SVD}_{r}(\LoS)$. Using Lemma~\ref{lem:incPres}, $\Lo_{i} = \LoS \VoS w_i$ for some $\norm{w_i} \leq \frac{33}{32} \mu \sqrt{\frac{r}{n}}$. Now consider the vector $y_i \coloneqq \US\SigS\VS^\top \VoS w_i$. Note that $\norm{(\SigS)^{-1} \US^\top y_i} \leq \frac{33}{32} \mu \sqrt{\frac{r}{n}}$. Now, using definition of $\Pr{\mathcal {E}}$: 
  \begin{align*}
  	&\norm{\Pr{\US}\lprp{\Lo_{i} + \No_{i}} - \Pr{\mathcal{E}}\lprp{\Lo_{i} + \No_{i}}} \leq \norm{\Pr{\US}\lprp{\Lo_{i} + \No_{i}} - y_i} \\ 
  	&\leq \norm{\Pr{\US} (\No_{i})} + \norm{\Pr{\US} \lprp{\LoS \VoS w_i} - \US\SigS\VS^\top \VoS w_i} \\
  	&\overset{(\zeta_1)}{\leq} \norm{\Pr{\US} (\No_{i})} + \frac{33}{32} \mu \sqrt{\frac{r}{n}} \norm{\Pr{\US}\lprp{\LoS}\VoS - \Pr{\US}\lprp{\LoS + \NoS} \VoS} \\
    &\leq \norm{\Pr{\US} (\No_{i})} + \frac{33}{32} \mu \sqrt{\frac{r}{n}} \norm{\Pr{\US} \lprp{\NoS}},
  \end{align*}
  where $(\zeta_1)$ holds from Lemma \ref{lem:incPres} and the fact that $\VoS$ has zeros in the rows corresponding to the support of $\Co$.
\end{proof}

\subsection{Proof of Lemma \ref{lem:invGau}}
\label{sec:pfInvGau}
\begin{proof}
	 Note that under the conditions of Theorem~\ref{thm:thmG}, the following three conditions hold with probability at least $1 - \delta$. 
   \begin{align*}    
    \#i \lprp{\norm{\No_i} \geq \sigma \lprp{\sqrt{d} + 2d^{\frac{1}{4}} \sqrt{\log\lprp{\frac{\mu^2 r}{c_2}}}}} &\leq \frac{3c_2n}{\mu^2 r} \\
    \forall \norm{v} = 1\qquad \#i \lprp{\abs{\iprod{v}{\No_i}} \geq 2 \sigma \sqrt{\log \lprp{\frac{\mu^2 r^2 d}{c_1}}}} &\leq \frac{6c_1n}{\mu^2 r^2 d} \\
    0.9\sigma\sqrt{n} \leq \norm{\No} &\leq 1.1\sigma\sqrt{n} 
  \end{align*}
  Also, for $c_1 = \frac{1}{12288}$ and $c_2 = \frac{1}{1536}$, the invariant implies at most $\frac{n}{512 \mu^2 r}$ clean data points are thresholded at any stage. This allows us to apply Lemma~\ref{lem:clExpGau} in the subsequent steps.

  We will prove the lemma by induction on the number of thresholding steps completed so far. Let $t$ denote the number of thresholding steps executed so far. 
	
	\textbf{Base Case (t = 0): } The invariant trivially holds before any data points have been thresholded.
	
	\textbf{Induction Step (t = k + 1): } Assuming that the invariant holds after the $k^{th}$ thresholding step, we now have two cases for the $(k + 1)^{th}$ thresholding step:
	\begin{enumerate}[leftmargin=*,labelindent=\labelwidth+\labelsep]
		\item[\textit{Case 1:}] Thresholding with respect to $\zeta_1$. For a column $i \not\in \supp{\Co}$ thresholded with respect to $\zeta_1$, we have from Lemma \ref{lem:clExpGau} that: 
		\begin{multline}
      \label{eqn:lpts}
			\norm{\No_{i}} \geq \zeta_1 - \frac{33}{32} \mu \sqrt{\frac{r}{n}} \norm{\No} = \sigma \lprp{\frac{5}{4} \mu \sqrt{r} + d^{\frac{1}{2}} + 2d^{\frac{1}{4}} \sqrt{\log\lprp{\frac{\mu^2 r}{c_2}}}} - \frac{5}{4} \sigma \mu \sqrt{r} \\
      \geq \sigma \lprp{d^{\frac{1}{2}} + 2d^{\frac{1}{4}} \sqrt{\log\lprp{\frac{\mu^2 r}{c_2}}}}.
		\end{multline}
		From our choice of $\zeta_1$, there are only $\frac{n}{1024 \mu^2 r}$ clean points which satisfy \ref{eqn:lpts}.
		\item[\textit{Case 2:}] Thresholding with respect to $\zeta_2$. For a column $i \not\in \supp{\Co}$, it can only be thresholded if the number of columns to be thresholded exceeds $\frac{24c_1 n}{\mu^2 r d}$. Note that:
    \begin{equation*}
      \norm{\Pr{U^{(k + 1)}}(\No_{i})} \geq \zeta_2 - \frac{5}{4} \sigma\mu\sqrt{r} = 2 \sigma \sqrt{2r} \sqrt{\log \lprp{\frac{\mu^2r^2d}{c_1}}}.
    \end{equation*}
    Note that $U^{(k + 1)}$ is at most a rank $2r$ subspace. Therefore, $\exists j$ such that: 
    \begin{equation}
      \label{eqn:prpts}
      \abs{\iprod{\No_i}{U^{(k + 1)}_j}} \geq 2\sigma \sqrt{\log \lprp{\frac{\mu^2r^2d}{c_1}}}.
    \end{equation}
    Taking a union bound over all $j \in [r + 1]$ and using Lemma~\ref{lem:gauipConc} for both positive and negative inner product values, we get that at most $\frac{12c_1 n}{\mu^2 rd}$ clean points satisfy \ref{eqn:prpts}. Since, we threshold at least $\frac{24c_1 n}{\mu^2 rd}$ points, at least half of them must be outliers and hence the invariant holds in the next iteration.
	\end{enumerate} 
\end{proof}

\section{Gaussian Noise: Proof of Theorem \ref{thm:thmG}}
\label{sec:pfThmGauss}
\begin{proof}
	We will prove the theorem for $c_1 = \frac{1}{12288}$ and $c_2 = \frac{1}{1536}$. For our choices of $c_1$ and $c_2$ and $n$, we have that:
  \begin{align*}    
    \#i \lprp{\norm{\No_i} \geq \sigma \lprp{\sqrt{d} + 2d^{\frac{1}{4}} \sqrt{\log\lprp{\frac{\mu^2 r}{c_2}}}}} &\leq \frac{3c_2n}{\mu^2 r} \\
    \forall \norm{v} = 1\qquad \#i \lprp{\abs{\iprod{v}{\No_i}} \geq 2 \sigma \sqrt{\log \lprp{\frac{\mu^2 r^2 d}{c_1}}}} &\leq \frac{6c_1n}{\mu^2 r^2 d} \\
    0.9\sigma\sqrt{n} \leq \norm{\No} &\leq 1.1\sigma\sqrt{n} 
  \end{align*}
  with probability at least $1 - \delta$ from Lemmas~\ref{lem:gauipConc}, \ref{lem:gauLthConc} and Corollary \ref{cor:probNrmBnd} along with our choice of $n$.

	From Lemma \ref{lem:invGau}, we know that Invariant \ref{inv:incMnt} holds at the termination of the algorithm. Therefore, at most $\frac{n}{512\mu^2r}$ inliers are removed (The number of inliers removed is at most $\colSp{}n + \frac{n}{1024\mu^2 r}$). 

  Suppose the algorithm terminated in the $T^{th}$ iteration. Let $M \coloneqq \Mo - \Ct[T]$. We will start by making a few observations. The algorithm terminates when no data point is thresholded. Let $[\U, \Sigma, \V] = \mathcal{SVD}_{r + 1}(M)$. Furthermore, define $\mathcal{E} = \{x: x = U\Sigma y \text{ for some } \norm{y} \leq 2\mu \sqrt{r / n}\}$. Now, define set $A$ as:
	
	\begin{equation*}
	A \coloneqq \left\{ i : \norm{\Pr{U}((M)_{i}) - \Pr{\mathcal{E}}((M)_{i})} \geq \sigma\sqrt{2r} \lprp{\frac{5}{4} \mu + 2\log^{\frac{1}{2}} \lprp{\frac{\mu^2 r^2 d}{c_1}}} \right\},
	\end{equation*}
	and $B$ as:
	\begin{equation*}
	B \coloneqq \left\{i: \norm{\Pr{U}(M_{i}) - \Pr{\mathcal{E}}(M_{i})} \geq \sigma \lprp{d^{\frac{1}{2}} + 2 d^{\frac{1}{4}} \lprp{\log^{\frac{1}{2}} \lprp{\frac{\mu^2 r}{c_2}}}} + \sigma \frac{5}{4} \mu \sqrt{r}\right\}.
	\end{equation*}
	
	Recall that we will threshold the columns in $A$ if $\abs{A} \geq \frac{24c_1 n}{\mu^2 r d}$ and the columns in $B$ if $B$ is not empty. Therefore, we know that:	
	\begin{equation}\label{eq:thmG0}
	\forall i \in [n] \norm{\Pr{U}(M_{i}) - \Pr{\mathcal{E}}(M_{i})} \leq \sigma \lprp{d^{\frac{1}{2}} + 2 d^{\frac{1}{4}} \lprp{\log^{\frac{1}{2}} \lprp{\frac{\mu^2 r}{c_2}}}} + \sigma \frac{5}{4} \mu \sqrt{r}, \ \abs{A} \leq \frac{24c_1n}{\mu^2 r d}.
	\end{equation}
	
	Let $S$ denote the set of data points that have been thresholded when the algorithm terminated, i.e $S = \cs^{(T)}$. Let $L = \LoS$, $N = \NoS$ and $C = \CoS$. Additionally, let $[\UoS, \SigoS, \VoS] = \mathcal{SVD}(L)$. Similar to the proofs of Theorems~\ref{thm:thmCln} and \ref{thm:thmN}, we start as follows:
	\begin{align}
	&\norm{\pperpur( \Lo)} \leq \lprp{\norm{\pperpur( L)}^2 + \sum\limits_{i \in S} \norm{\pperpur( \UoS \SigoS w_i)}^2}^{\frac{1}{2}}\notag\\
	& \overset{\zeta_1}{\leq} \lprp{\norm{\pperpur( L)}^2 + \sum\limits_{i \in S} \frac{9}{8} \mu^2 \frac{r}{n} \norm{\pperpur( \UoS \SigoS)}^2}^{\frac{1}{2}} \overset{\zeta_2}{\leq} \lprp{\norm{\pperpur( L)}^2 + 3 \nThresh n\frac{9}{8} \mu^2 \frac{r}{n} \norm{\pperpur( L)}^2}^{\frac{1}{2}} \notag\\
	& \leq \frac{5}{4} \norm{\pperpur( L)} \leq \frac{5}{4} \lprp{\norm{N} + \norm{\pperpur( L + N)}}\notag\\
	& \overset{\zeta_3}{\leq} \frac{5}{4} \lprp{\norm{\No} + \norm{\pperpur( M)}} = \frac{5}{4} \lprp{\norm{\No} + \norm{\pperpur( \Pr{U}(M))}},\label{eq:thmG1}
	\end{align}
	$\zeta_1$ follows using Lemma~\ref{lem:incPres}, $\zeta_2$ follows using $|S|\leq 2\rho n$, $\zeta_3$ follows using Lemma~\ref{lem:disSing}
	where the last equality follows from the fact that $U$ consists of the top $r + 1$ singular vectors of $M$.
	
	Now, let $Y$ be an orthogonal basis of the subspace spanned by $\Pr{U} (L)$. Note that the subspace spanned by $Y$ is at most rank-$r$. Let $O$ denote the set of corrupted columns that haven't been thresholded at the termination of the algorithm and let $O_l \coloneqq O \cap A$ and $O_s \coloneqq O \backslash O_l$. We can now bound $\norm{\pperpur( \Pr{U}(M))}$ as follows:
	\begin{align}
	\norm{\pperpur( \Pr{U}(M))} & \leq \norm{\pperpy( \Pr{U}(L + N + C))} \leq \norm{\pperpy( \Pr{U}(N + C))} \leq \norm{N} + \norm{\pperpy( \Pr{U}(C))} \notag\\
	& \leq \norm{\No} + \lprp{\underbrace{\sum\limits_{i \in O_l} \norm{\pperpy( \Pr{U}\lprp{C_{i}})}^2}_{\text{Term 1}} + \underbrace{\sum\limits_{j \in O_s} \norm{\pperpy( \Pr{U}\lprp{C_{j}})}^2}_{\text{Term 2}}}^{\frac{1}{2}},\label{eq:thmG2}
	\end{align}
	where first inequality follows from the fact that $U_{1:r}$ are top singular vectors of $\Pr{U}(M)$ and second inequality follows from definition of $Y$. 
	
	We can now bound Term 1 as follows: 
	\begin{align*}
	& \sum\limits_{i \in O_l} \norm{\pperpy( \Pr{U}\lprp{C_{i}})}^2 \overset{(\zeta_1)}{\leq} 2\sum\limits_{i \in O_l} \norm{\pperpy( \Pr{\mathcal{E}}\lprp{C_{i}})}^2 + \norm{\pperpy( \lprp{\Pr{U}\lprp{C_{i}} - \Pr{\mathcal{E}}\lprp{C_{i}}})}^2 \\
	& \leq 2\sum\limits_{i \in O_l} \frac{4\mu^2 r}{n} \norm{\pperpy( U\Sigma)}^2 + \norm{\Pr{U}\lprp{C_{i}} - \Pr{\mathcal{E}}\lprp{C_{i}}}^2\hfill (\text{From Definition of } \mathcal{E})\\
	& \overset{(\zeta_2)}{\leq} \frac{48c_1 n}{\mu^2 r d} \lprp{\frac{4\mu^2 r}{n} \norm{\pperpy( U\Sigma)}^2 + \sigma^2 \lprp{4\mu^2 r + 2\lprp{\sqrt{d} + 2d^{\frac{1}{4}}\sqrt{\log \lprp{\frac{\mu^2 r}{c_2}}}}^2}} \\
	& \overset{(\zeta_3)}{\leq} \frac{48c_1 n}{\mu^2 r d} \lprp{\frac{4\mu^2 r}{n} \norm{\pperpy( U\Sigma)}^2 + \sigma^2 \lprp{4\mu^2 r + 2\lprp{2d + 8d^{\frac{1}{2}}\log \lprp{\frac{\mu^2 r}{c_2}}}}} \\
	& \leq \frac{48c_1 n}{\mu^2 r d} \lprp{\frac{4\mu^2 r}{n} \norm{\pperpy( U\Sigma)}^2 + \sigma^2 \lprp{4\mu^2 r + 4d + 16d^{\frac{1}{2}}\log \lprp{\frac{\mu^2 r}{c_2}}}} \\
	& \overset{(\zeta_4)}{\leq} \frac{48c_1 n}{\mu^2 r d} \lprp{\frac{4\mu^2 r}{n} \norm{\pperpy( U\Sigma)}^2 + \sigma^2 \lprp{4\mu^2 r + 4d + 16d^{\frac{1}{2}}\log(\mu^2r) + 16d^{\frac{1}{2}}\log(1536)}} \\
	& \overset{(\zeta_5)}{\leq} \frac{48c_1 n}{\mu^2 r d} \lprp{\frac{4\mu^2 r}{n} \norm{\pperpy( U\Sigma)}^2 + \sigma^2 \lprp{24\mu^2 rd + 120d}} \\
	& \leq \frac{\norm{\pperpy( M)}^2}{32} + \frac{48 c_1 n}{\mu^2 r d} \sigma^2 \lprp{144 \mu^2 r d} \hfill (\text{From Lemma~\ref{lem:disSing}}) \\
	& \leq \frac{\norm{\pperpy( M)}^2}{32} + 0.563 \sigma^2 n
	\end{align*}
	where $(\zeta_1)$ and $(\zeta_3)$ follow from $(a + b)^2 \leq 2a^2 + 2b^2$, $\zeta_2$ follows from \eqref{eq:thmG0}.  $(\zeta_4)$ and $(\zeta_5)$ follow from $\log (x) \leq x $ and $\sqrt{x} \leq x$ for all $x \geq 1$.
	
	We now bound Term 2 as:
	\begin{align*}
	& \sum\limits_{j \in O_s} \norm{\pperpy( \Pr{U}\lprp{C_{j}})}^2 \overset{(\zeta_6)}{\leq} 2\sum\limits_{j \in O_s} \norm{\pperpy( \Pr{\mathcal{E}}\lprp{C_{j}})}^2 + \norm{\pperpy( \lprp{\Pr{U}\lprp{C_{j}} - \Pr{\mathcal{E}}\lprp{C_{j}}})}^2 \\
	& \leq 2\sum\limits_{i \in O_s} \frac{4\mu^2 r}{n} \norm{\pperpy( U\Sigma)}^2 + \norm{\Pr{U}\lprp{C_{j}} - \Pr{\mathcal{E}}\lprp{C_{j}}}^2 \hfill (\text{Definition of $\mathcal{E}$})\\
	& \overset{(\zeta_7)}{\leq} 2 \colSp n \lprp{\frac{4\mu^2 r}{n} \norm{\pperpy( U\Sigma)}^2 + 2 \sigma^2 \lprp{4\mu^2 r + 8 r \log \lprp{\frac{\mu^2 r^2 d}{c_1}}}} \\
	& \overset{(\zeta_8)}{\leq} \frac{\norm{\pperpy( M)}^2}{32} + 4 \colSp n \sigma^2 \lprp{4 \mu^2 r + 8r\log (\mu^2) + 16r\log(r) + 8r\log(d) + 8r \log(12288)} \\
	& \overset{(\zeta_9)}{\leq} \frac{\norm{\pperpy( M)}^2}{32} + 4 \colSp n \sigma^2 \lprp{12 \mu^2 r + 24r\log(d) + 8r \log(12288)} \overset{(\zeta_{10})}{\leq} \frac{\norm{\pperpy( M)}^2}{32} + 0.435 \sigma^2 n \log(d),
	\end{align*}
	where $(\zeta_6)$ and $(\zeta_7)$ follow from $(a + b)^2 \leq 2a^2 + 2b^2$ along with \eqref{eq:thmG0}. $(\zeta_8)$ follows from Lemma~\ref{lem:disSing}, $(\zeta_9)$ follows from the fact that $r \leq d$ and $\log(x) \leq x$ and $(\zeta_{10})$ follows from assuming $\log (d) \geq 1$ and $\mu \geq 1$.
	
	From our bounds on Term 1 and Term 2 and \eqref{eq:thmG2}, we have:
	\begin{equation*}
	\norm{\pperpur( \Pr{U} (M))} \leq \norm{\pperpy( \Pr{U}(M))} \leq \frac{4}{3} \sigma \sqrt{n \log(d)}.
	\end{equation*}
	Theorem now follows by using the above observation with \eqref{eq:thmG1}: 
	\begin{equation*}
	\norm{\pperpur( \Lo)} \leq \frac{6}{4} \sigma \sqrt{n} + \frac{5}{3} \sigma \sqrt{n \log(d)} \leq 4 \sigma\sqrt{n \log(d)}.
	\end{equation*}
\end{proof}


\section{\ncpcab}
\label{sec:torpb}
In this section, we propose an improvement to Algorithm~\ref{alg:ncpca} which uses binary search instead of a linear scan in the outer iteration. This improves the running time on Algorithm~\ref{alg:ncpca} by almost a factor of $r$.

\subsection{Algorithm}
\label{sec:algToprb}
In this section, we present our algorithm (See Algorithm~\ref{alg:ncpcab}) for \orpcan{} which improves the running time of Algorithm~\ref{alg:ncpca} by almost a factor of $r$. The main insight is that inner iteration of Algorithm~\ref{alg:ncpca} is independent of the value of $k$ in the outer iteration save for the rank of the projection. In Algorithm~\ref{alg:ncpcab}, we use binary search on $k$ instead of a linear scan which reduces the number of outer iterations from $\order{r}$ to $\order{\log r}$.

\begin{algorithm}[!ht]
  \caption{Binary search based TORP (\ncpcab)}
  \label{alg:ncpcab}
  \begin{algorithmic}[1]
    \STATE \textbf{Input}: Corrupted matrix $\Mo \in \mathbb{R}^{d\times n}$, Target rank $r$, Expressivity parameter $\thresh$, Threshold fraction $\nThresh$, Number of inner iterations $T$
    \STATE $minK \leftarrow 1$, $maxK \leftarrow r$

    \WHILE{$minK \leq maxK$}
    \STATE $k \leftarrow \floor{\frac{minK + maxK}{2}}$
    \STATE $\Ct[0] \leftarrow 0$, $\tau \leftarrow false$
      \FOR{$t = 0$ to $t = T$}
        \STATE $[\Uot{t}, \Sigot{t}, \Vot{t}] \leftarrow \mathcal{SVD}_k\lprp{\Mo - \Ct[t]}$, $\Lt[t] \leftarrow \Uot{t}\Sigot{t}(\Vot{t})^\top$ \hspace{1em}\rlap{\smash{$\left.\begin{array}{@{}c@{}}\\{}{}\end{array} \right\}%
        \begin{tabular}{c}Projection onto space of \\ low rank matrices \end{tabular}$}}
        \vspace*{1.5em}
        \STATE $E \leftarrow (\Sigot{t})^{-1} (\Uot{t})^\top \Mo$ {\em /* Compute Incoherence */}
        \STATE $R \leftarrow (I - \Uot{t}(\Uot{t})^\top) \Mo$ {\em /* Compute residual */} 
        \hspace{5.7em}\rlap{\smash{$\left.\begin{array}{@{}c@{}}\\{}\\{}\\{}\end{array} \right\}%
        		\begin{tabular}{c}Projection onto space of \\ column sparse matrices \end{tabular}$}}
        \STATE $\cs^{(t + 1)} \leftarrow \mathcal{HT}_{2\nThresh} \lprp{\Mo, E} \cup \mathcal{HT}_\nThresh \lprp{\Mo, R}$ 
        \STATE $\Ct[t + 1] \leftarrow \Mo_{\cs^{(t + 1)}}$

        \STATE $\nOut \leftarrow \abs{\{i: \norm{E_{i}} \geq \thresh\}}$ {\em /* Compute high incoherence points */}
        \STATE $\tau \leftarrow \tau \vee (\nOut \geq 2\nThresh n)$ {\em /* Check termination conditions */}
      \ENDFOR
      \IF{$\tau$}
        \STATE $maxK \leftarrow k - 1$
      \ELSE
      	\STATE $minK \leftarrow k + 1$
      	\STATE $[U, \Sigma, V] \leftarrow \mathcal{SVD}_k\lprp{\Mo - \Ct[T + 1]}$
      \ENDIF
    \ENDWHILE

    \STATE \textbf{Return: }$U$
  \end{algorithmic}
\end{algorithm}

\subsection{Analysis}
\label{sec:anncpcab}
In this section, we will state and prove a theoretical guarantee for Algorithm~\ref{alg:ncpcab}.

\begin{theorem}
  Assume the conditions of Theorem~\ref{thm:thmN}. Furthermore, suppose that $\frob{\No} \leq \frac{\sigma_{k}(\Lo)}{16}$ for some $k$. Then, for all $\alpha\leq \frac{1}{128 \mu^2 r}$, Algorithm~\ref{alg:ncpcab} when run with parameters $\rho=\frac{1}{128\mu^2 r}$, $\eta=2\mu \sqrt{\frac{r}{n}}$ and $T = \log{\frac{20n\sigma_1(M)}{\epsilon}}$, returns a subspace $U$ which satisfies:
  \begin{equation*}
    \frob{(I - UU^\top) \Lo} \leq 3 \frob{(I - \Uo_{1:k}(\Uo_{1:k})^\top) \Lo} + 9 \frob{\No} + \frac{\epsilon}{10n}.  
  \end{equation*}
\end{theorem}

\begin{proof}
  We will begin by bounding the running time of the algorithm. Note that because of the binary search, the algorithm will run for at most $\order{\log r}$ outer iterations. 

  Let $t$ denote the number of outer iterations of the algorithm. Let the value of $k$($maxK$, $minK$) in iteration $t$ be denoted by $k^{(t)}$ ($maxK^{(t)}$, $minK^{(t)}$). We will first prove the claim that at any point in the running of the algorithm, $maxK \geq k$. We will prove the claim via induction on the number of iterations $t$:
  \begin{enumerate}
    \item[] Base Case: $t = 0$ The base case is trivially true as $maxK = r$.
    \item[] Induction Step: $t = l + 1$ Assume that the claim remains true at iteration $t = l$. In the $(l + 1)^{th}$ iteration, we assume two cases:
    \begin{enumerate}
      \item[] Case 1: The inner iteration finishes with $\tau = false$. In this case, $maxK$ is not updated. So, the claim remains true for $t = (l + 1)$
      \item[] Case 2: The inner iteration finishes with $\tau = true$. In this case, $k^{(t)} > k$ (From Lemma~\ref{lem:clExpGen} and the termination condition of the inner iteration.). In this iteration, $maxK$ is updated to $k^{(t)} - 1 \geq k$. Thus, the claim remains true.
    \end{enumerate}
  \end{enumerate}

  Therefore, at termination of the algorithm, we have $maxK \geq k$. Suppose that the algorithm terminated after iteration $T$. Note that $minK^{(t)} \leq k \leq maxK^{(t)} \forall 0\leq t \leq T$. Therefore, we have $minK^{(T + 1)} = maxK^{(T + 1)} + 1$. For this to happen, the inner iteration must have run with $k^{(T^\prime)} = maxK^{(T + 1)}$ with $\tau = false$ for some iteration $T^\prime$ and also that this is the last such successful iteration as $minK$ is not updated after iteration $T^\prime$. Therefore, the algorithm returns the subspace corresponding to $k^{(T^\prime)} = maxK^{(T + 1)} \geq k$. Since the inner iteration is successful for iteration $T^\prime$, the Theorem is true from the application of Lemma~\ref{lem:inItPerfGen} and noting that $k^{T^\prime} \geq k$.
\end{proof}

\section{Fast Projection Operator}
\label{sec:fpo}
In this section, we will describe a fast algorithm to compute the projection operator onto the ellipsoid in Algorithm~\ref{alg:ncpca}. Formally, we are provided an orthogonal basis $U \in \mathbb{R}^{d \times r}$, a positive diagonal matrix $\Sigma$, a bound $b$ and a vector $x$. Let $\mathcal{E} = \{y: y = U\Sigma z \text{ for some } \norm{z} \leq b\}$. The goal is to compute the projection of the vector $x$ onto the set $\mathcal{E}$.

\subsection{Algorithm}
\label{sec:fpoalg}
In this section, we present our algorithm (Algorithm~\ref{alg:fpr}) to compute the projection onto the set $\mathcal{E}$. We show that the projection operation boils down to an univariate optimization problem on a monotone function. We then perform binary search on an interval in which the solution is guaranteed to lie. 

\begin{algorithm}[!ht]
  \caption{$w ~ =$ \fpr$(U, \Sigma, b, x, \epsilon)$}
  \label{alg:fpr}
  \begin{algorithmic}[1]
    \STATE \textbf{Input}: Orthogonal Basis $U \in \mathbb{R}^{d \times r}$, Positive Diagonal Matrix $\Sigma$, Bound $b$, Projection Vector $x$, Accuracy Parameter $\epsilon$
	
	\STATE $\sigma_{min} = \min\limits_{i \in [r]} (\Sigma_{i, i})$, $\sigma_{max} = \max\limits_{i \in [r]} (\Sigma_{i, i})$
	\STATE $y \leftarrow \Sigma U^\top x$

	\STATE $\lambda_{min} = 0$, $\lambda_{max} = \frac{\norm{y}}{b}$
	\STATE $T \leftarrow \log\lprp{\frac{\lambda_{max} \sqrt{r} \norm{x}}{\sigma^2_{min} \epsilon}}$

	\FOR{Iteration $t = 0$ to $t = T$}
		\STATE $\lambda^{(t)} \leftarrow \frac{\lambda_{min} + \lambda_{max}}{2}$
		\STATE $z^{(t)} \leftarrow  (\lambda I + \Sigma^2)^{-1} y$
		\IF{$\norm{z^{(t)}} \leq b$}
			\STATE $\lambda_{max} \leftarrow \lambda^{(t)}$
		\ELSE
			\STATE $\lambda_{min} \leftarrow \lambda^{(t)}$
		\ENDIF
	\ENDFOR

    \STATE \textbf{Return: }$U\Sigma z^{(T)}$
  \end{algorithmic}
\end{algorithm}

\subsection{Analysis}
\label{sec:fpoan}

\begin{theorem}
	\label{thm:fpo}
	Let $U \in \mathbb{R}^{d \times r}$ be an orthonormal matrix and $\Sigma \in \mathbb{R}^{r \times r}$ be a positive diagonal matrix. Then, for any $b \geq 0$, $x \in \mathbb{R}^{d}$ and $\epsilon$, the vector $w$ returned by Algorithm~\ref{alg:fpr} satisfies:
	\begin{equation*}
		\norm{w - \Pr{\mathcal{E}}(x)} \leq \epsilon
	\end{equation*}
	where $\mathcal{E} \coloneqq \{y: y = U\Sigma z \text{ for some } \norm{z} \leq b\}$
\end{theorem}

\begin{proof}
	We first define the convex optimization problem corresponding to the projection operator $\Pr{\mathcal{E}}$. Then, we have:
	\begin{equation*}
		\Pr{\mathcal{E}} (x) = \argmin\limits_{y} \norm{x - y}\ s.t\ y \in \mathcal{E}
	\end{equation*}
	
	Since, $y \in \mathcal{E}$, a solution to the above optimization problem is equivalent to:
	
	\begin{equation}
		\label{eqn:fpropt}
		\Pr{\mathcal{E}} (x) = \argmin\limits_{z} \norm{x - U\Sigma z}^2\ s.t\ \norm{z}^2 \leq b^2
	\end{equation}
	
	Note that both the constraint and the objective function are convex. Therefore, we can introduce a KKT multiplier $\lambda \geq 0$ and writing down the stationarity conditions of \ref{eqn:fpropt}, we get:
	
	\begin{equation*}
		2 \Sigma^2 z + 2\lambda z = 2\Sigma U^\top x \implies z = \lprp{\Sigma^2 + \lambda I}^{-1} \Sigma U^\top x
	\end{equation*}
	
	Now, we just need to ensure that $\norm{\lprp{\Sigma^2 + \lambda I}^{-1} \Sigma U^\top x} \leq b$. Let $f(\lambda) = \norm{\lprp{\Sigma^2 + \lambda I}^{-1} \Sigma U^\top x}$ for $\lambda \geq 0$. Let $\lambda^*$ be the solution to $f(\lambda) = \min (\norm{\Sigma^{-1} U^\top x}, b)$. We will first prove that at any point in the running of the algorithm $\lambda_{max} \geq \lambda^*$ and $\lambda_{min} \leq \lambda^*$. We prove the claim by induction on the number of iterations $t$:
	\begin{enumerate}
		\item[] \textbf{Base Case $t = 0$}: Since $\lambda_{min} = 0$, the lower bound holds trivially. That $\lambda_{max} \geq \lambda^*$ can be proved as follows:
		\begin{equation*}
			f(\lambda_{max}) = \norm{\lprp{\Sigma^2 + \lambda_{max} I}^{-1} \Sigma U^\top x} \leq \min\lprp{\norm{\Sigma^{-1} U^\top x}, \frac{\norm{\Sigma U^\top x}}{\lambda_{max}}} \leq \min\lprp{\norm{\Sigma^{-1} U^\top x}, b}.
		\end{equation*}
		Since, $f$ is a monotonically decreasing function, the claim holds true in the base case.
		\item[] \textbf{Induction Step $t = (k + 1)$}: Assume that the claim holds till $t = k$. We have two cases for iteration $k + 1$:
		\begin{enumerate}
			\item[] Case 1: $\lambda_{max} \leftarrow \lambda^{(t + 1)}$. In this case, $\lambda_{min} \leq \lambda^*$ still holds from the inductive hypothesis. For $\lambda_{max}$, we have: 
			\begin{equation*}
			\hspace*{-50pt}	f(\lambda^{(t + 1)}) = \norm{\lprp{\Sigma^2 + \lambda^{(t + 1)} I}^{-1} \Sigma U^\top x} \leq \min\lprp{\norm{\Sigma^{-1} U^\top x}, f(\lambda^{(t + 1)})} \leq \min\lprp{\norm{\Sigma^{-1} U^\top x}, b},
			\end{equation*}
			where the last inequality holds from the fact that $\lambda_{max}$ was updated in this iteration. From the monotonicity of $f$, the induction hypothesis holds in this iteration. 
			\item[] Case 2: $\lambda_{min} \leftarrow \lambda^{(t + 1)}$. In this case, $\lambda_{max} \geq \lambda^*$ by the inductive hypothesis. In this case, we have:
			\begin{equation*}
				f(\lambda^{(t + 1)}) \geq b \geq f(\lambda^*).
			\end{equation*}
			From the monotonicity of $f$, the induction hypothesis holds in this iteration.
		\end{enumerate}
	\end{enumerate}

	Note that $(\lambda_{max} - \lambda_{min})$ is halved at each iteration. Therefore, at the termination of the algorithm, we have $(\lambda_{max} - \lambda_{min}) \leq \frac{\sigma_{min}^2 \epsilon}{\sqrt{r}\norm{x}}$. From our claim, this implies that $\abs{\lambda^* - \lambda^{(T)}} \leq \frac{\sigma_{min}^2 \epsilon}{\sqrt{r}\norm{x}}$. Note that we can write $\Pr{\mathcal{E}}(x) = U \Sigma^2 (\lambda^* I + \Sigma^2)^{-1} U^\top x$. Note that $\norm{\Pr{\mathcal{E}}(x) - w} = \norm{U^\top(\Pr{\mathcal{E}}(x) - w)}$. We will now bound the element-wise difference between $U^\top\Pr{\mathcal{E}}(x)$ and $U^\top w$:

	\begin{multline*}
		\abs{e_i^\top \Sigma^2 ((\lambda^* I + \Sigma^2)^{-1} - (\lambda^{(T)} I + \Sigma^2)^{-1}) U^\top x} \leq \norm{x} \abs{\sigma^2_i \lprp{\frac{1}{\lambda^* + \sigma^2_i} - \frac{1}{\lambda^{(T)} + \sigma^2_i}}} \\ 
		\leq \norm{x} \sigma^2_i \abs{\frac{\lambda^{(T)} - \lambda^*} {(\lambda^* + \sigma^2_i)(\lambda^{(T)} + \sigma^2_i)}} \leq \norm{x} \abs{\frac{\lambda^{(T)} - \lambda^*}{\sigma^2_i}} \leq \frac{\epsilon}{\sqrt{r}}.
	\end{multline*}

	By applying the element-wise bound to $\Sigma^2 ((\lambda^* I + \Sigma^2)^{-1} - (\lambda^{(T)} I + \Sigma^2)^{-1}) U^\top x$, we have:
	\begin{equation*}
		\norm{\Pr{\mathcal{E}}(x) - w} \leq \sqrt{r} \frac{\epsilon}{\sqrt{r}} \leq \epsilon.
	\end{equation*}
\end{proof}

\end{document}